\documentclass[twocolumn,american,english,twocolumn, cleanfoot, cleanhead]{asme2e}
\usepackage[T1]{fontenc}
\usepackage[latin1]{inputenc}
\usepackage{xcolor}
\usepackage{mathtools}
\usepackage{multirow}
\usepackage{amsmath}
\usepackage{amssymb}
\usepackage{graphicx}

\makeatletter

\providecommand{\tabularnewline}{\\}

\usepackage{graphicx}
\usepackage{import}
\usepackage{makecell}
\graphicspath{{Images/}}

\newtheorem{lemma}{Lemma}[section]

\newtheorem{definition}[lemma]{Definition}

\newtheorem{remark}{Remark}

\usepackage{algorithm,algpseudocode}
\algnewcommand{\LineComment}[1]{\State \(\triangleright\) #1}

\definecolor{BLUE}{RGB}{0,0,255}
\usepackage{caption}

\makeatother

\usepackage{babel}
\begin{document}
\global\long\def\dq#1{\underline{\boldsymbol{#1}}}%

\global\long\def\quat#1{\boldsymbol{#1}}%

\global\long\def\mymatrix#1{\boldsymbol{#1}}%

\global\long\def\myvec#1{\boldsymbol{#1}}%

\global\long\def\mapvec#1{\boldsymbol{#1}}%

\global\long\def\dualvector#1{\underline{\boldsymbol{#1}}}%

\global\long\def\dual{\varepsilon}%

\global\long\def\dotproduct#1{\langle#1\rangle}%

\global\long\def\norm#1{\left\Vert #1\right\Vert }%

\global\long\def\mydual#1{\underline{#1}}%

\global\long\def\getp#1{\operatorname{\mathcal{P}}\left(#1\right)}%

\global\long\def\getd#1{\operatorname{\mathcal{D}}\left(#1\right)}%

\global\long\def\hamilton#1#2{\overset{#1}{\operatorname{\mymatrix H}}\left(#2\right)}%

\global\long\def\hamiquat#1#2{\overset{#1}{\operatorname{\mymatrix H}}_{4}\left(#2\right)}%

\global\long\def\hami#1{\overset{#1}{\operatorname{\mymatrix H}}}%

\global\long\def\tplus{\dq{{\cal T}}}%

\global\long\def\swap#1{\text{swap}\{#1\}}%

\global\long\def\imi{\hat{\imath}}%

\global\long\def\imj{\hat{\jmath}}%

\global\long\def\imk{\hat{k}}%

\global\long\def\real#1{\operatorname{\mathrm{Re}}\left(#1\right)}%

\global\long\def\imag#1{\operatorname{\mathrm{Im}}\left(#1\right)}%

\global\long\def\imvec{\boldsymbol{\imath}}%

\global\long\def\vector{\operatorname{vec}}%

\global\long\def\mathpzc#1{\fontmathpzc{#1}}%

\global\long\def\cost#1#2{\underset{\text{#2}}{\operatorname{\mathfrak{C}}}\left(\ensuremath{#1}\right)}%

\global\long\def\diag#1{\operatorname{diag}\left(#1\right)}%

\global\long\def\frame#1{\mathcal{F}_{#1}}%

\global\long\def\ad#1#2{\mathrm{Ad}\left(#1\right)#2}%

\global\long\def\spin{\text{Spin}(3)}%

\global\long\def\spinr{\text{Spin}(3){\ltimes}\mathbb{R}^{3}}%

\global\long\def\spinrlie{\text{spin}(3){\ltimes}\mathbb{R}^{3}}%

\global\long\def\norm#1{\left\Vert #1\right\Vert }%

\global\long\def\minim#1#2{ \begin{aligned} &  \underset{#1}{\min}  &   &  #2 \end{aligned}
 }%

\global\long\def\minimone#1#2#3{ \begin{aligned} &  \underset{#1}{\min}  &   &  #2 \\
  &  \text{subject to}  &   &  #3 
\end{aligned}
 }%

\global\long\def\minimtwo#1#2#3#4{ \begin{aligned} &  \underset{#1}{\min}  &   &  #2 \\
  &  \text{subject to}  &   &  #3 \\
  &   &   &  #4 
\end{aligned}
 }%

\global\long\def\minimtwom#1#2#3#4{ \begin{aligned} &  \underset{#1}{\min}  &   &  #2 \\
  &  \text{subject to}  &   &  #3 \\
  &   &   &  #4 
\end{aligned}
 }%

\global\long\def\minimthree#1#2#3#4#5{ \begin{aligned} &  \underset{#1}{\min}  &   &  #2 \\
  &  \text{subject to}  &   &  #3 \\
  &   &   &  #4 \\
  &   &   &  #5 
\end{aligned}
 }%

\global\long\def\argmin#1#2#3#4{ \begin{aligned}#4  &  \underset{#1}{\arg\min}  &   &  #2\\
  &  \text{subject to}  &   &  #3 
\end{aligned}
 }%

\title{Dynamics of Mobile Manipulators using Dual Quaternion Algebra}
\author{$\begin{array}{cc}
\text{\textbf{Frederico} \textbf{Fernandes} \textbf{Afonso} \textbf{Silva},}^{\dagger}\,\text{\textbf{Juan} \textbf{Jos}}\acute{\text{\textbf{e}}}\,\text{\textbf{Quiroz}-\textbf{Oma}}\tilde{\text{\textbf{n}}}\text{\textbf{a},}^{\ddagger} & \quad\text{\textbf{Bruno} \textbf{Vilhena} \textbf{Adorno}}\\
\text{Graduate Program in Electrical Engineering} & \quad\text{Department of Electrical and Electronic Engineering}\\
\text{Federal University of Minas Gerais (UFMG)} & \quad\text{The University of Manchester}\\
\text{Av. Ant\^{o}nio Carlos 6627, 31270-901, Belo Horizonte-MG, Brazil} & \quad\text{Sackville Street, Manchester M13 9PL, United Kingdom}\\
\text{Email: fredf.afonso@gmail.com,}^{\dagger}\,\text{juanjqogm@gmail.com}^{\ddagger} & \quad\text{Email: bruno.adorno@manchester.ac.uk}
\end{array}$}
\maketitle
\begin{abstract}
This paper presents two approaches to obtain the dynamical equations
of mobile manipulators using dual quaternion algebra. The first one
is based on a general recursive Newton-Euler formulation and uses
twists and wrenches, which are propagated through high-level algebraic
operations and works for any type of joints and arbitrary parameterizations.
The second approach is based on Gauss's Principle of Least Constraint
(GPLC) and includes arbitrary equality constraints. In addition to
showing the connections of GPLC with Gibbs-Appell and Kane's equations,
we use it to model a nonholonomic mobile manipulator. Our current
formulations are more general than their counterparts in the state
of the art, although GPLC is more computationally expensive, and simulation
results show that they are as accurate as the classic recursive Newton-Euler
algorithm.

\end{abstract}

\section{Introduction}

In the last thirty years, there have been an expressive amount of
papers dealing with different representations for robot modeling.
Notorious examples can be found in the works of Featherstone \cite{Featherstone2008Book,Featherstone2010,Featherstone2010a},
McCarthy \cite{McCarthyBook1990,Dooley1991,Perez2004}, Selig \cite{Selig2004,Selig2005},
and Bayro-Corrochano \cite{Selig2010}, among many others.

One of the reasons for such investigations is that the complexity
of a robotic system goes far beyond the complexity of the mechanism
itself. A typical robotic system involves motion/force/impedance control,
path planning, task planning, and many more higher-level layers. Therefore,
representations that are very useful for robot modeling, such as homogeneous
transformation matrices, not necessarily are easy to use when performing
pose control or impedance control, for example \cite{Yuan1988}. This
is precisely the reason why it is common to use homogeneous transformation
matrices to obtain the robot kinematics but then indirectly find the
geometric Jacobian and, finally, to use quaternions and position vectors
to perform pose control in the task-space \cite{Xian2004}. There
are several drawbacks in using the aforementioned strategy. The mix
of different representations unnecessarily complicates the overall
representation and the mapping between those different representations
usually introduces mathematical artifacts, such as algorithmic singularities
and discontinuities.

In contrast, elements of dual quaternion algebra have strong geometrical
meaning, such as in screw theory, and are also represented as coupled
entities within single elements. In kinematics, this representation
has been extensively explored to obtain the robot kinematics and differential
kinematics \cite{Perez2004,Adorno2011,Gouasmi2012,Cohen2016,Ozgur2016,Kong2017,Dantam2020}.
Furthermore, in recent works, dual quaternions have been used to perform
admittance control \cite{Fonseca2020}, which is fundamental in physical
human-robot interaction; constrained motion control \cite{Marinho2019,Quiroz-Omana2019},
which takes into account geometrical constraints imposed by the workspace;
hybrid control, which takes into account the topology of the space
of rigid motions \cite{Kussaba2017} and optimal control, which uses
a linear-quadratic optimal tracking controller for robotic manipulators
\cite{Marinho2015}; distributed pose formation control \cite{Savino2020}
and cooperative manipulation \cite{Adorno2010,Figueredo2014IROS},
including the ones that involve human-robot collaboration \cite{Adorno2015};
and to define high-level geometrical tasks \cite{Lana2015}.

Furthermore, elements such as unit dual quaternions and pure dual
quaternions, when equipped with standard multiplication and addition
operations, form Lie groups with associated Lie algebras. Therefore,
a formulation based on dual quaternion algebra offers the geometrical
insights of screw theory, the rigor of Lie Algebra, and a simple algebraic
treatment of the dynamical model as in the spatial algebra \cite{Featherstone2008Book},
often reducing the necessity of an extensive geometric analysis of
the mechanism, which contrasts with approaches based on the matrix
representation of screw theory \cite{Huang2015,Renda2017}.

Some works have used dual quaternion algebra to describe rigid body
dynamics over the last decades \cite{Yang1964,Yang1966,Yang1967,Yang1971},
although not necessarily creating a general formalism for multibody
system analysis. Among the works that have sought such formalism,
most are based on three-dimensional dual vectors and demand some mapping
to higher dimensional vectors to obtain the system dynamic equations
\cite{Pennock1983,Dooley1991,Shoham1993,Valverde2018a}, therefore
losing the elegance and compactness of an analysis based only on dual
quaternion algebra and, at times, incurring in abuses of notation
\cite{Dooley1991} or demanding artificial swaps on the vectors \cite{Valverde2018}
to deal with the mixing of representations. Other works focused on
the propagation of dual quaternions \cite{Hachicho2000} or the computational
aspects of algorithms based on dual quaternion algebra \cite{MirandadeFarias2019},
rather than on the algebraic and geometrical insights that the algebra
provides when dealing with more complex robots (e.g., nonholonomic
mobile manipulators) and more general types of joints (e.g., helical,
cylindrical, $6$-DoF, etc.).

In conclusion, since there is no method based on dual quaternion algebra
that adequately encompasses the dynamic model of mobile manipulators
and general types of joints, there is still a theoretical gap that
creates an unnecessary need for intermediate mappings when using higher-level
algorithms based on dual quaternions to connect them to the low-level
dynamic model. The purpose of this paper is to fill that gap by proposing
a suitable dynamic model of mobile manipulators with arbitrary joints
using dual quaternion algebra.

\subsection{Statement of contributions}

This paper presents two approaches to obtain the dynamical equations
of serial manipulators using dual quaternion algebra. The first one
is based on the the recursive Newton-Euler formulation and the second
one applies the Gauss's Principle of Least Constraint to obtain the
dynamical model of a serial mobile manipulators subject to nonholonomic
constraints. The contributions of this paper to the state-of-the-art
are the following:
\begin{enumerate}
\item A systematic procedure to obtain the recursive equations for the dynamic
model of mobile manipulators using dual quaternion algebra and the
Newton-Euler formalism, which has linear cost on the number of links.
This approach simplifies the classic procedure by removing the necessity
of exhaustive geometrical analyses because wrenches and twists are
propagated through high-level algebraic operations. Compared to previous
works, our approach is more general because it works for arbitrary
types of joints and we do not impose any particular parameterization
convention for the propagation of twists;
\item A closed-form for the dynamic model of serial manipulators based on
the Gauss's Principle of Least Constraint (GPLC) and dual quaternion
algebra. We impose additional constraints in the GPLC formulation
to model nonholonomic mobile manipulators. We apply the fundamental
equation proposed by Udwadia-Kalaba \cite{FirdausE.UdwadiaandRobertE.Kalaba1992},
which employs a simpler method, albeit equivalent, than Lagrange multipliers
to enforce the equality constraints. In addition, we present the skew
symmetry property related to the inertia and Coriolis matrices, which
is paramount when designing passivity-based controllers. Finally,
we show the connections of the Gauss's Principle of Least Constraint
with the Gibbs-Appell and Kane's equations using our formulation based
on dual quaternion algebra.
\end{enumerate}
We validate the proposed algorithms in simulation using three different
robots: a fixed-base $50$-DoF serial manipulator, a $9$-DoF holonomic
mobile manipulator, and an $8$-DoF nonholonomic mobile manipulator.
Moreover, we compare our results with the ones provided by a realistic
simulator, and with an implementation of the state of the art. Furthermore,
we present the computational costs of the proposed methodologies.

This paper is organized as follows: Section~\ref{sec:Mathematical-Preliminaries}
presents a brief mathematical background on dual quaternion algebra;
Section~\ref{sec:Newton-Euler-Model} introduces a general Newton-Euler
formulation based on dual quaternion algebra, whereas Section~\ref{sec:Gauss-Principle}
introduces the dual quaternion formulation based on the Gauss's Principle
of Least Constraints; Section~\ref{sec:Results} presents both the
validation of the proposed methodologies through simulations and their
computational costs; finally, Section~\ref{sec:Conclusions} gives
the final remarks and points to further research directions.

\section{Mathematical Preliminaries \label{sec:Mathematical-Preliminaries}}

Dual quaternions \cite{Selig2005} are elements of the set
\begin{align}
\mathcal{H} & \triangleq\{\quat h_{\mathcal{P}}+\dual\quat h_{\mathcal{D}}\::\:\quat h_{\mathcal{P}},\quat h_{\mathcal{D}}\in\mathbb{H},\,\dual\neq0,\,\dual^{2}=0\},\label{eq:dq_def}
\end{align}
where
\begin{align}
\mathbb{H} & \triangleq\{h_{1}+\imi h_{2}+\imj h_{3}+\imk h_{4}\::\:h_{1},h_{2},h_{3},h_{4}\in\mathbb{R}\}\label{eq:quat_def}
\end{align}
is the set of quaternions, in which $\imi$, $\imj$ and $\imk$ are
imaginary units with the properties $\imi^{2}=\imj^{2}=\imk^{2}=\imi\imj\imk=-1$\cite{Hamilton1844}.
Addition and multiplication of dual quaternions are analogous to their
counterparts of real and complex numbers. One must only respect the
properties of the dual unit $\dual$ and imaginary units $\imi,\imj,\imk$.

Given $\dq h\in\mathcal{H}$ such that
\[
\dq h=\underbrace{h_{1}+\imi h_{2}+\imj h_{3}+\imk h_{4}}_{\quat h_{\mathcal{P}}}+\dual\underbrace{\left(h_{5}+\imi h_{6}+\imj h_{7}+\imk h_{8}\right)}_{\quat h_{\mathcal{D}}},
\]
the operators $\mathcal{P}\left(\dq h\right)\triangleq\quat h_{\mathcal{P}}$
and $\mathcal{D}\left(\dq h\right)\triangleq\quat h_{\mathcal{D}}$
provide the primary part and dual part of $\dq h$, respectively,
whereas the operators $\real{\dq h}\triangleq h_{1}+\dual h_{5}$
and $\imag{\dq h}=\imi h_{2}+\imj h_{3}+\imk h_{4}+\dual\left(\imi h_{6}+\imj h_{7}+\imk h_{8}\right)$
provide the real and the imaginary part of $\dq h$, respectively.
The conjugate of $\dq h$ is defined as $\dq h^{*}\triangleq\real{\dq h}-\imag{\dq h}$
and its norm is given by $\norm{\dq h}=\sqrt{\dq h\dq h^{*}}=\sqrt{\dq h^{*}\dq h}$.

The subset $\dq{\mathcal{S}}=\left\{ \dq h\in\mathcal{H}:\norm{\dq h}=1\right\} $
of unit dual quaternions is used to represent poses (position and
orientation) in the three-dimensional space and form the group $\spinr$
under the multiplication operation.\footnote{The symbol $\ltimes$ represents the semi-direct product between groups
\cite[p. 22]{Selig2005}.} Any $\dq x\in\dq{\mathcal{S}}$ can always be written as $\dq x=\quat r+\dual\left(1/2\right)\quat p\quat r$,
where $\quat p=\imi x+\imj y+\imk z$ represents the position $\left(x,y,z\right)$
in the three-dimensional space and $\quat r=\cos\left(\phi/2\right)+\quat n\sin\left(\phi/2\right)$
represents a rotation, in which $\phi\in[0,2\pi)$ is the rotation
angle around the rotation axis$\quat n\in\mathbb{H}_{p}\cap\mathbb{S}^{3}$,
with $\mathbb{H}_{p}\triangleq\left\{ \quat h\in\mathbb{H}:\real{\quat h}=0\right\} $
and $\mathbb{S}^{3}=\left\{ \quat h\in\mathbb{H}:\norm{\quat h}=1\right\} $\cite{Selig2005}.

Given the set $\mathcal{H}_{p}=\left\{ \dq h\in\mathcal{H}:\real{\dq h}=0\right\} $
of pure dual quaternions, which are used to represent twists and wrenches,
the operator $\mathrm{Ad}:\dq{\mathcal{S}}\times\mathcal{H}_{p}\to\mathcal{H}_{p}$
performs rigid motions on those entities. For instance, given a twist
expressed in frame $\frame a$, namely $\dq{\xi}^{a}\in\mathcal{H}_{p}$,
and the unit dual quaternion $\dq x_{a}^{b}$ that gives the pose
of $\frame a$ with respect to $\frame b$, the same twist is expressed
in frame $\frame b$ as\footnote{Notice that superscripts represent the original frame, and subscripts
represent the modified frame. This convention of subscripts and superscripts
is maintained throughout this paper. If no superscript is used, we
assume the global inertial frame.}
\begin{equation}
\dq{\xi}^{b}=\ad{\dq x_{a}^{b}}{\dq{\xi}^{a}}=\dq x_{a}^{b}\dq{\xi}^{a}\left(\dq x_{a}^{b}\right)^{*}.\label{eq:adj_transf}
\end{equation}

The time derivative of $\dq x_{b}^{a}$ is given by \cite{Adorno2017}
\begin{equation}
\dot{\dq x}_{b}^{a}=\frac{1}{2}\dq{\xi}_{ab}^{a}\dq x_{b}^{a}=\frac{1}{2}\dq x_{b}^{a}\dq{\xi}_{ab}^{b},\label{eq:x_dot}
\end{equation}
where 
\begin{gather}
\dq{\xi}_{ab}^{a}=\quat{\omega}_{ab}^{a}+\dual\left(\dot{\quat p}_{ab}^{a}+\quat p_{ab}^{a}\times\quat{\omega}_{ab}^{a}\right)\label{eq:twist-inertial-frame}
\end{gather}
is the twist of frame $\frame b$ with respect to frame $\frame a$,
expressed in frame $\frame a$,\footnote{\textcolor{blue}{I}t is important to use three indices here because
the twist between two frames can be seen from a third frame. So, for
example, $\dq{\xi}_{a,b}^{c}$ is the twist of frame $\frame b$ with
respect to frame $\frame a$, expressed in frame $\frame c$. The
same interpretation is used for wrenches, which will be properly introduced
in Section~\ref{subsec:dqNE_backward_recursion}).} with $\quat{\omega}_{ab}^{a}\in\mathbb{H}_{p}$ being the angular
velocity, and 
\begin{gather}
\dq{\xi}_{ab}^{b}=\ad{\dq x_{a}^{b}}{\dq{\xi}_{ab}^{a}}=\quat{\omega}_{ab}^{b}+\dual\dot{\quat p}_{ab}^{b}\label{eq:twist-body-frame}
\end{gather}
is the twist expressed in $\frame b$. Furthermore, $\dq{\xi}_{ab}^{a}$
and $\dq{\xi}_{ab}^{b}$ are elements of the Lie algebra associated
with $\spinr$. Additionally,
\begin{align}
\quat p\times\quat{\omega}\triangleq & \frac{\quat p\quat{\omega}-\quat{\omega}\quat p}{2},\label{eq:cross_product}
\end{align}
$\quat p,\quat{\omega}\in\mathbb{H}_{p}$, is the cross-product between
pure quaternions, which is analogous to the cross product between
vectors in $\mathbb{R}^{3}$ \cite{Adorno2017}.

The cross-product between $\dq l,\dq s\in\mathcal{H}_{p}$, where
$\dq l=\quat l+\dual\quat l'$ and $\dq s=\quat s+\dual\quat s'$,
is analogous to \eqref{eq:cross_product} and given by
\begin{equation}
\dq l\times\dq s\triangleq\frac{\dq l\dq s-\dq s\dq l}{2}=\quat l\times\quat s+\dual\left(\quat l\times\quat s'+\quat l'\times\quat s\right).\label{eq:cross_product_dq}
\end{equation}

\begin{lemma}\label{lem:derivative_of_adjoint_transformation}

If $\dq x\in\mathcal{\dq S},$ such that $\dot{\dq x}=(1/2)\dq{\xi}\dq x$
and $\dq{\xi}'\in\mathcal{H}_{p}$, then 
\begin{equation}
\frac{d}{dt}\left(\ad{\dq x}{\dq{\xi}'}\right)=\ad{\dq x}{\dot{\dq{\xi}'}}+\dq{\xi}\times\left(\ad{\dq x}{\dq{\xi}'}\right).\label{eq:derivative_of_adjoint_transformation}
\end{equation}

\end{lemma}

\begin{proof}

Using \eqref{eq:adj_transf}, \eqref{eq:x_dot}, and the fact that
$\left(\dq{\xi}\dq x\right)^{*}=-\dq x^{*}\dq{\xi}$, we obtain 
\begin{align}
\frac{d}{dt}\left(\ad{\dq x}{\dq{\xi}}'\right) & =\dot{\dq x}\dq{\xi}'\dq x^{*}+\dq x\dot{\dq{\xi}'}\dq x^{*}+\dq x\dq{\xi}'\dot{\dq x}^{*}\nonumber \\
 & =\frac{1}{2}\dq{\xi}\left(\dq x\dq{\xi}'\dq x^{*}\right)+\dq x\dot{\dq{\xi}'}\dq x^{*}-\frac{1}{2}\left(\dq x\dq{\xi}'\dq x^{*}\right)\dq{\xi}.\label{eq:adjoint_derivative_intermediate}
\end{align}
Finally, using \eqref{eq:cross_product_dq} in \eqref{eq:adjoint_derivative_intermediate}
yields \eqref{eq:derivative_of_adjoint_transformation}.

\end{proof}

The quaternionic inertia tensor is defined as
\begin{equation}
\quat{\mathbb{I}}\triangleq\left(\quat i_{x},\quat i_{y},\quat i_{z}\right)\in\mathbb{H}_{p}^{3}\subset\mathcal{H}^{n},\label{eq:quat_inertia_tensor}
\end{equation}
where $\quat i_{x}=I_{xx}\imi+I_{xy}\imj+I_{xz}\imk$, $\quat i_{y}=I_{yx}\imi+I_{yy}\imj+I_{yz}\imk$,
and $\quat i_{z}=I_{zx}\imi+I_{zy}\imj+I_{zz}\imk$, in which $I_{nn}$,
with $n\in\left\{ x,y,z\right\} $, are elements of the rigid body's
inertia tensor.

\begin{definition}\label{def_operator_l}

Given $\quat A=\left(\quat a_{x},\quat a_{y},\quat a_{z}\right)\in\mathbb{H}_{p}^{3}$
and $\quat b\in\mathbb{H}_{p}$, the operator $\mathcal{L}_{3}:\mathbb{H}_{p}^{3}\times\mathbb{H}_{p}\to\mathbb{H}_{p}$,
is defined as
\begin{align}
\mathcal{L}_{3}\left(\quat A\right)\quat b & =\imi\dotproduct{\quat a_{x},\quat b}+\imj\dotproduct{\quat a_{y},\quat b}+\imk\dotproduct{\quat a_{z},\quat b},\label{eq:operator_l}
\end{align}
where $\dotproduct{\cdot,\cdot}:\mathbb{H}_{p}\to\mathbb{R}$ is the
inner product between quaternions;\footnote{The inner product in $\mathbb{H}_{p}$ is equivalent to the inner
product in $\mathbb{R}^{3}$.} that is, given $\quat a,\quat b\in\mathbb{H}_{p}$, then $\dotproduct{\quat a,\quat b}\triangleq-(\quat{ab}+\quat{ba})/2$.

\end{definition}

From Definition~\ref{def_operator_l}, it follows that the angular
momentum $\quat{\ell}$ in (dual) quaternion algebra is given by
\begin{equation}
\quat{\ell}=\mathcal{L}_{3}\left(\quat{\mathbb{I}}\right)\quat{\omega},\label{eq:angular_momentum}
\end{equation}
where $\quat{\omega}\in\mathbb{H}_{p}$ is the angular velocity. Direct
calculation shows that \eqref{eq:angular_momentum} is equivalent
to its counterpart in vector algebra.

Given the quaternionic inertia tensor $\quat{\mathbb{I}}'\in\mathbb{H}_{p}^{3}$
of a rigid body expressed in frame $\frame{}'$, and the rigid body's
angular velocity $\quat{\omega}\in\mathbb{H}_{p}$ expressed in frame
$\frame{}$, the angular momentum expressed in frame $\frame{}$ is
given by
\begin{equation}
\quat{\ell}=\ad{\quat r^{*}}{\mathcal{L}_{3}\left(\quat{\mathbb{I}}'\right)\ad{\quat r}{\quat{\omega}}},\label{eq:similarity-transformation}
\end{equation}
where $\quat r$ is the rotation quaternion from $\frame{}^{'}$ to
$\frame{}$. Eq.~\eqref{eq:similarity-transformation} is analogous
to the equation that one obtains when using similarity transformations
of rotation matrices and vectors in $\mathbb{R}^{3}$.

\selectlanguage{american}%

\section{Dual Quaternion Newton-Euler Model \label{sec:Newton-Euler-Model}}

This section presents the recurrence relations of the Newton-Euler
model using dual quaternion algebra for mobile manipulators with arbitrary
joints, assuming that the full kinematic model is available using
dual quaternion representation \cite{Adorno2011}.

For illustrative purposes, and without loss of generality, consider
the mobile manipulator shown in Fig.~\ref{fig:mobile_nDOF_robot},
composed of an $n_{\ell}$-DoF serial manipulator attached to a 3-DoF
mobile base. The goal is to find the wrenches $\dq{\Gamma}\in\mathcal{H}_{p}^{n}$
acting on the $n=n_{\ell}+1$ centers of mass (CoM) of the robot's
mobile base and the $n_{\ell}$ links, given the corresponding robot
configuration, generalized velocities, and generalized accelerations.
This can be seen as a function $\mathcal{N}\,:\,\mathbb{R}^{n_{\ell}+3}\times\mathbb{R}^{n_{\ell}+3}\times\mathbb{R}^{n_{\ell}+3}\rightarrow\mathcal{H}_{p}^{n}$,
where $n_{\ell}+3$ is the dimension of the configuration space and
$n$ is the number of rigid bodies in the kinematic chain (e.g., the
mobile base and links), such that
\begin{equation}
\dq{\Gamma}=\mathcal{N}\left(\myvec q,\dot{\myvec q},\ddot{\myvec q}\right).\label{eq:dq_NE_as_a_function}
\end{equation}
\begin{figure}[t]
\def\svgwidth{2.5\columnwidth}
\noindent \begin{centering}
{\Huge{}\resizebox{0.8\columnwidth}{!}{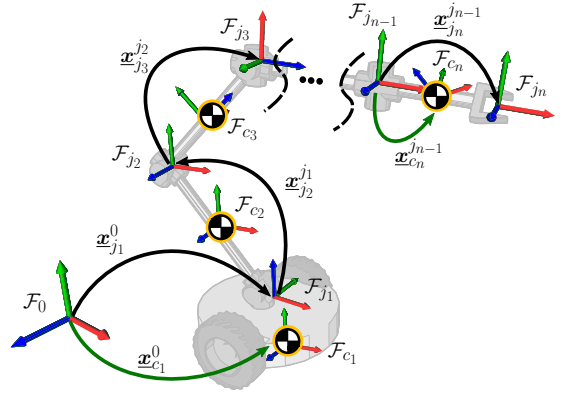}}{\Huge\par}
\par\end{centering}
\caption{Mobile manipulator composed of a manipulator with $n_{\ell}$-DoF
serially attached to a 3-DoF mobile base.\label{fig:mobile_nDOF_robot}}
\end{figure}

\subsection{Forward Recursion\label{subsec:Forward-Recursion}}

The first process of the algorithm consists of a serial sweeping of
the robot kinematic structure to calculate the twist of each CoM.
\footnote{Henceforth, we will use the expression ``twist of the CoM'' as a
shorthand for ``the twist of the frame attached to the CoM.''} The objective is to find the forward recurrence relations that will
then be used to iteratively obtain the wrenches acting on the robot's
mobile base and joints.

\subsubsection{Twists \label{subsec:Twists}}

The twist of the mobile base's CoM (i.e., the first CoM in the serial
kinematic chain) with respect to the inertial frame $\frame 0$ ,
expressed in frame $\frame{c_{1}}$, is given by the pure dual quaternion
\begin{align}
\dq{\xi}_{0,c_{1}}^{c_{1}} & =\quat{\omega}_{0,c_{1}}^{c_{1}}+\dual\quat v_{0,c_{1}}^{c_{1}},\label{eq:twist_c1_0}
\end{align}
where $\quat{\omega}_{0,c_{1}}^{c_{1}}=\omega_{x}\imi+\omega_{y}\imj+\omega_{z}\imk$
and $\quat v_{0,c_{1}}^{c_{1}}=v_{x}\imi+v_{y}\imj+v_{z}\imk$ are,
respectively, the angular and the linear velocities. The twist of
a holonomic mobile base is kinematically equivalent to the one of
a planar joint, shown in Table~\ref{tab:joint_twists-2}.

The twist of the first link's CoM (i.e., of the second rigid body
in the serial kinematic chain) with respect to the inertial frame
depends not only on the twist generated by its joint but also on the
twist of the mobile base because they are physically attached. Therefore,
\begin{align}
\dq{\xi}_{0,c_{2}}^{c_{2}} & =\dq{\xi}_{0,j_{1}}^{c_{2}}+\dq{\xi}_{j_{1},c_{2}}^{c_{2}},\nonumber \\
 & =\ad{\dq x_{c_{1}}^{c_{2}}}{\dq{\xi}_{0,j_{1}}^{c_{1}}}+\ad{\dq x_{j_{1}}^{c_{2}}}{\dq{\xi}_{j_{1},c_{2}}^{j_{1}}},\nonumber \\
 & =\ad{\dq x_{c_{1}}^{c_{2}}}{\left(\dq{\xi}_{0,c_{1}}^{c_{1}}+\dq{\xi}_{c_{1},j_{1}}^{c_{1}}\right)}+\ad{\dq x_{j_{1}}^{c_{2}}}{\dq{\xi}_{j_{1},c_{2}}^{j_{1}}},\label{eq:twist_c2_0}
\end{align}
where $\dq{\xi}_{j_{1},c_{2}}^{j_{1}}=\quat{\omega}_{j_{1},c_{2}}^{j_{1}}+\dual\quat v_{j_{1},c_{2}}^{j_{1}}$
is the twist on $\frame{c_{2}}$ generated by the first joint, and
$\dq{\xi}_{0,j_{1}}^{c_{2}}$ is the twist on $\frame{j_{1}}$ generated
by the mobile base, but expressed in $\frame{c_{2}}$ using a suitable
transformation as in \eqref{eq:adj_transf}. Table~\ref{tab:joint_twists-2}
presents the twists for different types of joints, where $\quat l\in\mathbb{H}_{p}\cap\mathbb{S}^{3}$
is a \emph{constant} unit-norm pure quaternion, which is equivalent
to a vector in $\mathbb{R}^{3}$, that is used to define an arbitrary
axis. For instance, when using the Denavit-Hartenberg (DH) convention,
$\quat l=\imk$, which is equivalent to the $z$-axis. Furthermore,
$\omega,\omega_{x},\omega_{y},\omega_{z}\in\mathbb{R}$ and $v,v_{x},v_{y},v_{z}\in\mathbb{R}$
are the scalar components of the angular and linear velocities, respectively.
Again, when using the DH convention, $\omega=\dot{\theta}$ for a
revolute joint and $v=\dot{d}$ for a prismatic joint. For helical
joints, the constant $h\in\mathbb{R}$ is called the \emph{pitch.}

\begin{table*}[t]
\begin{centering}
\caption{Twists of some of the most commonly used joints in robotics, where
$\protect\quat l\in\mathbb{H}_{p}\cap\mathbb{S}^{3}$ and $\omega,\omega_{x},\omega_{y},\omega_{z},v,v_{x},v_{y},v_{z},h\in\mathbb{R}$.\label{tab:joint_twists-2}}
\par\end{centering}
\centering{}%
\begin{tabular}{>{\centering}p{3.3cm}>{\centering}m{2.2cm}>{\centering}p{2.5cm}>{\centering}p{3cm}>{\centering}p{2.5cm}>{\centering}p{2.5cm}c}
\hline 
\noindent \centering{}6-DoF & \noindent \centering{}Revolute & \noindent \centering{}Spherical & \noindent \centering{}Cylindrical & \noindent \centering{}Planar & \noindent \centering{}Prismatic & Helical\tabularnewline
\hline 
\noindent \centering{}\resizebox{0.22\columnwidth}{!}{
\begingroup%
  \makeatletter%
  \providecommand\color[2][]{%
    \errmessage{(Inkscape) Color is used for the text in Inkscape, but the package 'color.sty' is not loaded}%
    \renewcommand\color[2][]{}%
  }%
  \providecommand\transparent[1]{%
    \errmessage{(Inkscape) Transparency is used (non-zero) for the text in Inkscape, but the package 'transparent.sty' is not loaded}%
    \renewcommand\transparent[1]{}%
  }%
  \providecommand\rotatebox[2]{#2}%
  \newcommand*\fsize{\dimexpr\f@size pt\relax}%
  \newcommand*\lineheight[1]{\fontsize{\fsize}{#1\fsize}\selectfont}%
  \ifx\svgwidth\undefined%
    \setlength{\unitlength}{2654.08008966bp}%
    \ifx\svgscale\undefined%
      \relax%
    \else%
      \setlength{\unitlength}{\unitlength * \real{\svgscale}}%
    \fi%
  \else%
    \setlength{\unitlength}{\svgwidth}%
  \fi%
  \global\let\svgwidth\undefined%
  \global\let\svgscale\undefined%
  \makeatother%
  \begin{picture}(1,0.98153764)%
    \lineheight{1}%
    \setlength\tabcolsep{0pt}%
    \put(0,0){\includegraphics[width=\unitlength,page=1]{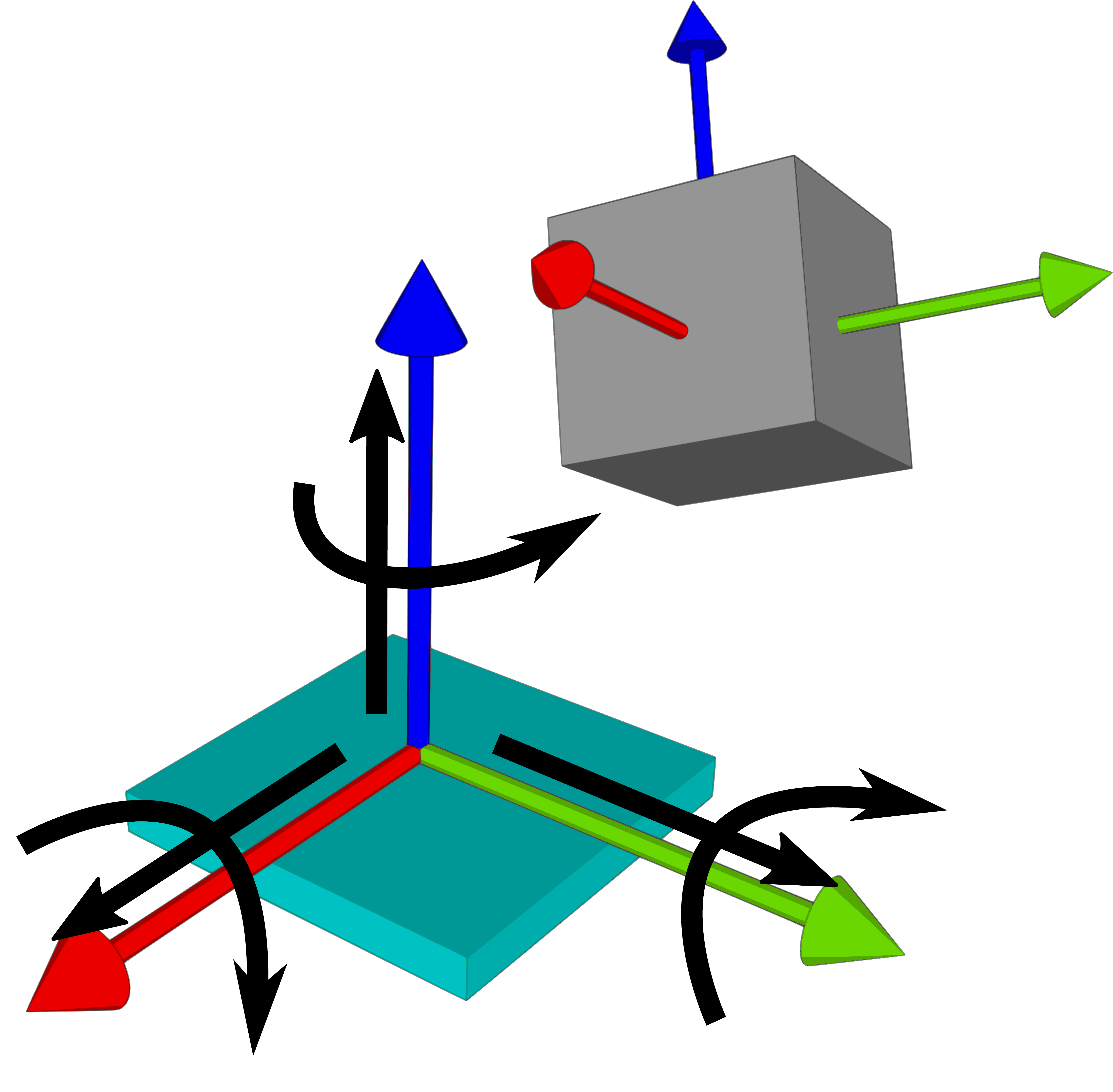}}%
  \end{picture}%
\endgroup%
} & \noindent \centering{}\resizebox{0.15\columnwidth}{!}{
\begingroup%
  \makeatletter%
  \providecommand\color[2][]{%
    \errmessage{(Inkscape) Color is used for the text in Inkscape, but the package 'color.sty' is not loaded}%
    \renewcommand\color[2][]{}%
  }%
  \providecommand\transparent[1]{%
    \errmessage{(Inkscape) Transparency is used (non-zero) for the text in Inkscape, but the package 'transparent.sty' is not loaded}%
    \renewcommand\transparent[1]{}%
  }%
  \providecommand\rotatebox[2]{#2}%
  \newcommand*\fsize{\dimexpr\f@size pt\relax}%
  \newcommand*\lineheight[1]{\fontsize{\fsize}{#1\fsize}\selectfont}%
  \ifx\svgwidth\undefined%
    \setlength{\unitlength}{1454.82150821bp}%
    \ifx\svgscale\undefined%
      \relax%
    \else%
      \setlength{\unitlength}{\unitlength * \real{\svgscale}}%
    \fi%
  \else%
    \setlength{\unitlength}{\svgwidth}%
  \fi%
  \global\let\svgwidth\undefined%
  \global\let\svgscale\undefined%
  \makeatother%
  \begin{picture}(1,1.42786475)%
    \lineheight{1}%
    \setlength\tabcolsep{0pt}%
    \put(0,0){\includegraphics[width=\unitlength,page=1]{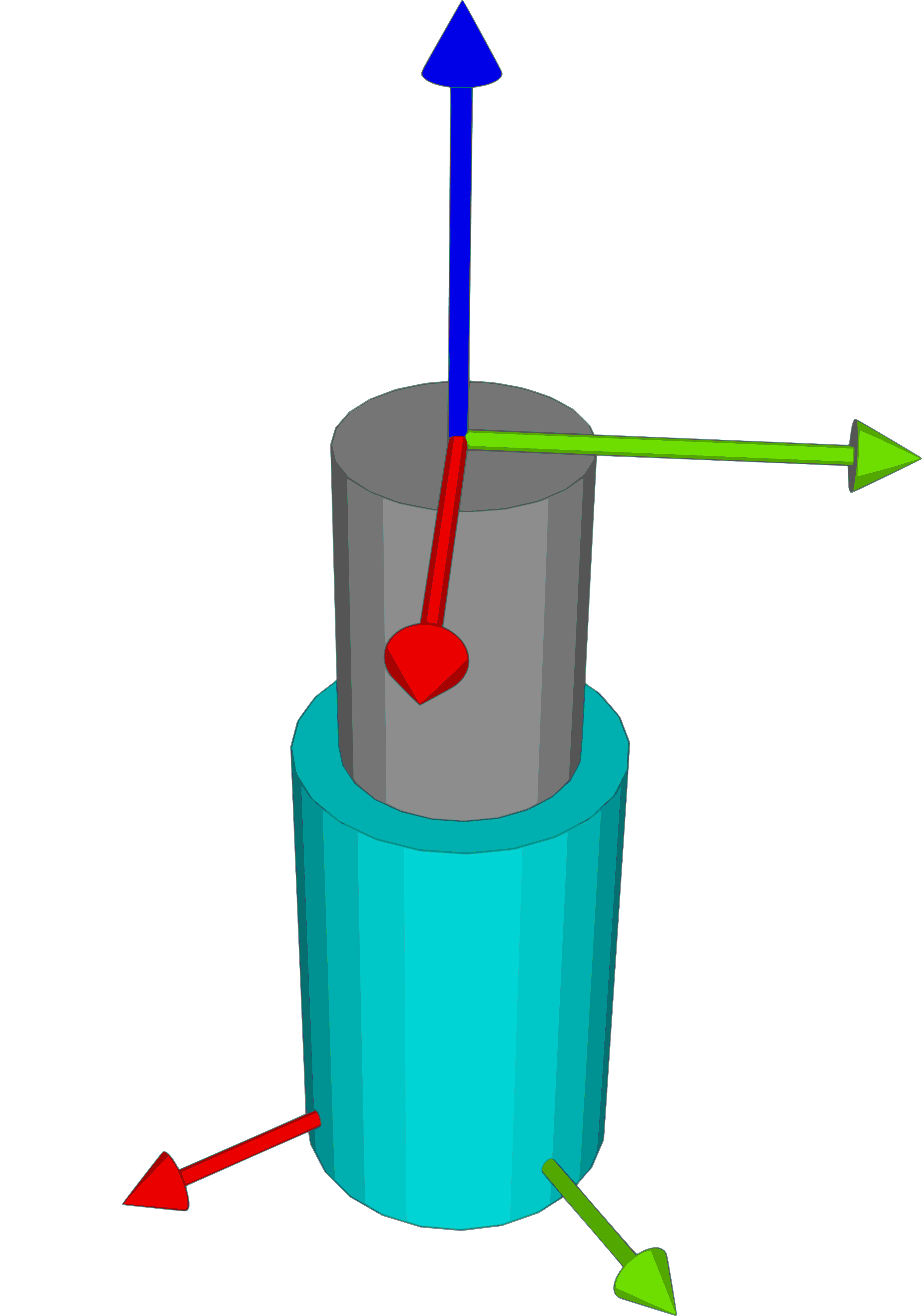}}%
    \put(0.66923827,1.06659749){\color[rgb]{0,0,0}\makebox(0,0)[lt]{\lineheight{1.25}\smash{\begin{tabular}[t]{l}$\theta$\end{tabular}}}}%
    \put(0,0){\includegraphics[width=\unitlength,page=2]{Images/joints/rotational.pdf}}%
  \end{picture}%
\endgroup%
} & \noindent \centering{}\resizebox{0.18\columnwidth}{!}{
\begingroup%
  \makeatletter%
  \providecommand\color[2][]{%
    \errmessage{(Inkscape) Color is used for the text in Inkscape, but the package 'color.sty' is not loaded}%
    \renewcommand\color[2][]{}%
  }%
  \providecommand\transparent[1]{%
    \errmessage{(Inkscape) Transparency is used (non-zero) for the text in Inkscape, but the package 'transparent.sty' is not loaded}%
    \renewcommand\transparent[1]{}%
  }%
  \providecommand\rotatebox[2]{#2}%
  \newcommand*\fsize{\dimexpr\f@size pt\relax}%
  \newcommand*\lineheight[1]{\fontsize{\fsize}{#1\fsize}\selectfont}%
  \ifx\svgwidth\undefined%
    \setlength{\unitlength}{2217.61334277bp}%
    \ifx\svgscale\undefined%
      \relax%
    \else%
      \setlength{\unitlength}{\unitlength * \real{\svgscale}}%
    \fi%
  \else%
    \setlength{\unitlength}{\svgwidth}%
  \fi%
  \global\let\svgwidth\undefined%
  \global\let\svgscale\undefined%
  \makeatother%
  \begin{picture}(1,0.90095711)%
    \lineheight{1}%
    \setlength\tabcolsep{0pt}%
    \put(0,0){\includegraphics[width=\unitlength,page=1]{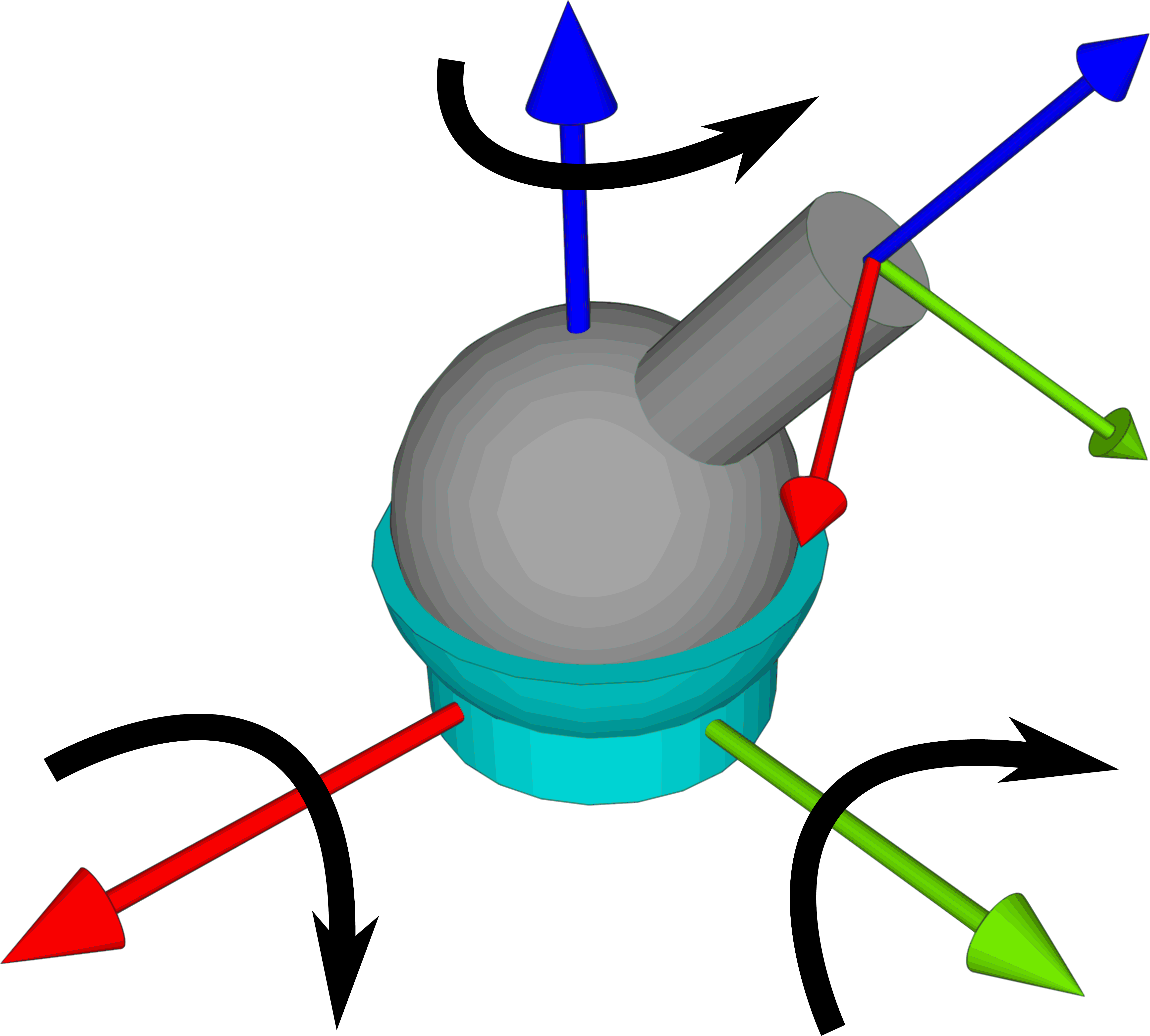}}%
  \end{picture}%
\endgroup%
} & \noindent \centering{}\resizebox{0.15\columnwidth}{!}{
\begingroup%
  \makeatletter%
  \providecommand\color[2][]{%
    \errmessage{(Inkscape) Color is used for the text in Inkscape, but the package 'color.sty' is not loaded}%
    \renewcommand\color[2][]{}%
  }%
  \providecommand\transparent[1]{%
    \errmessage{(Inkscape) Transparency is used (non-zero) for the text in Inkscape, but the package 'transparent.sty' is not loaded}%
    \renewcommand\transparent[1]{}%
  }%
  \providecommand\rotatebox[2]{#2}%
  \newcommand*\fsize{\dimexpr\f@size pt\relax}%
  \newcommand*\lineheight[1]{\fontsize{\fsize}{#1\fsize}\selectfont}%
  \ifx\svgwidth\undefined%
    \setlength{\unitlength}{688.13700999bp}%
    \ifx\svgscale\undefined%
      \relax%
    \else%
      \setlength{\unitlength}{\unitlength * \real{\svgscale}}%
    \fi%
  \else%
    \setlength{\unitlength}{\svgwidth}%
  \fi%
  \global\let\svgwidth\undefined%
  \global\let\svgscale\undefined%
  \makeatother%
  \begin{picture}(1,1.43517322)%
    \lineheight{1}%
    \setlength\tabcolsep{0pt}%
    \put(0,0){\includegraphics[width=\unitlength,page=1]{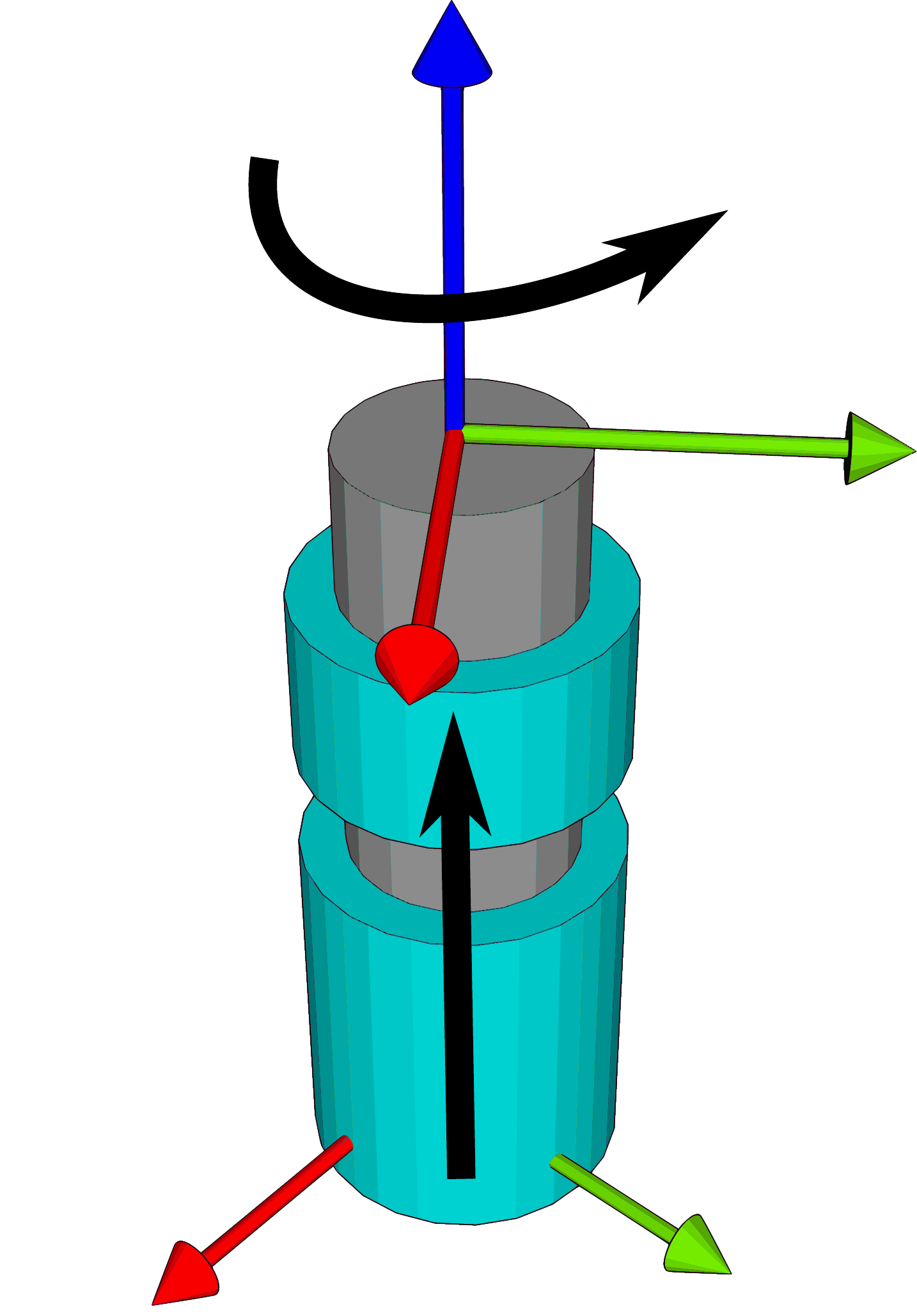}}%
  \end{picture}%
\endgroup%
} & \noindent \centering{}\resizebox{0.18\columnwidth}{!}{
\begingroup%
  \makeatletter%
  \providecommand\color[2][]{%
    \errmessage{(Inkscape) Color is used for the text in Inkscape, but the package 'color.sty' is not loaded}%
    \renewcommand\color[2][]{}%
  }%
  \providecommand\transparent[1]{%
    \errmessage{(Inkscape) Transparency is used (non-zero) for the text in Inkscape, but the package 'transparent.sty' is not loaded}%
    \renewcommand\transparent[1]{}%
  }%
  \providecommand\rotatebox[2]{#2}%
  \newcommand*\fsize{\dimexpr\f@size pt\relax}%
  \newcommand*\lineheight[1]{\fontsize{\fsize}{#1\fsize}\selectfont}%
  \ifx\svgwidth\undefined%
    \setlength{\unitlength}{2728.26392947bp}%
    \ifx\svgscale\undefined%
      \relax%
    \else%
      \setlength{\unitlength}{\unitlength * \real{\svgscale}}%
    \fi%
  \else%
    \setlength{\unitlength}{\svgwidth}%
  \fi%
  \global\let\svgwidth\undefined%
  \global\let\svgscale\undefined%
  \makeatother%
  \begin{picture}(1,0.79879242)%
    \lineheight{1}%
    \setlength\tabcolsep{0pt}%
    \put(0,0){\includegraphics[width=\unitlength,page=1]{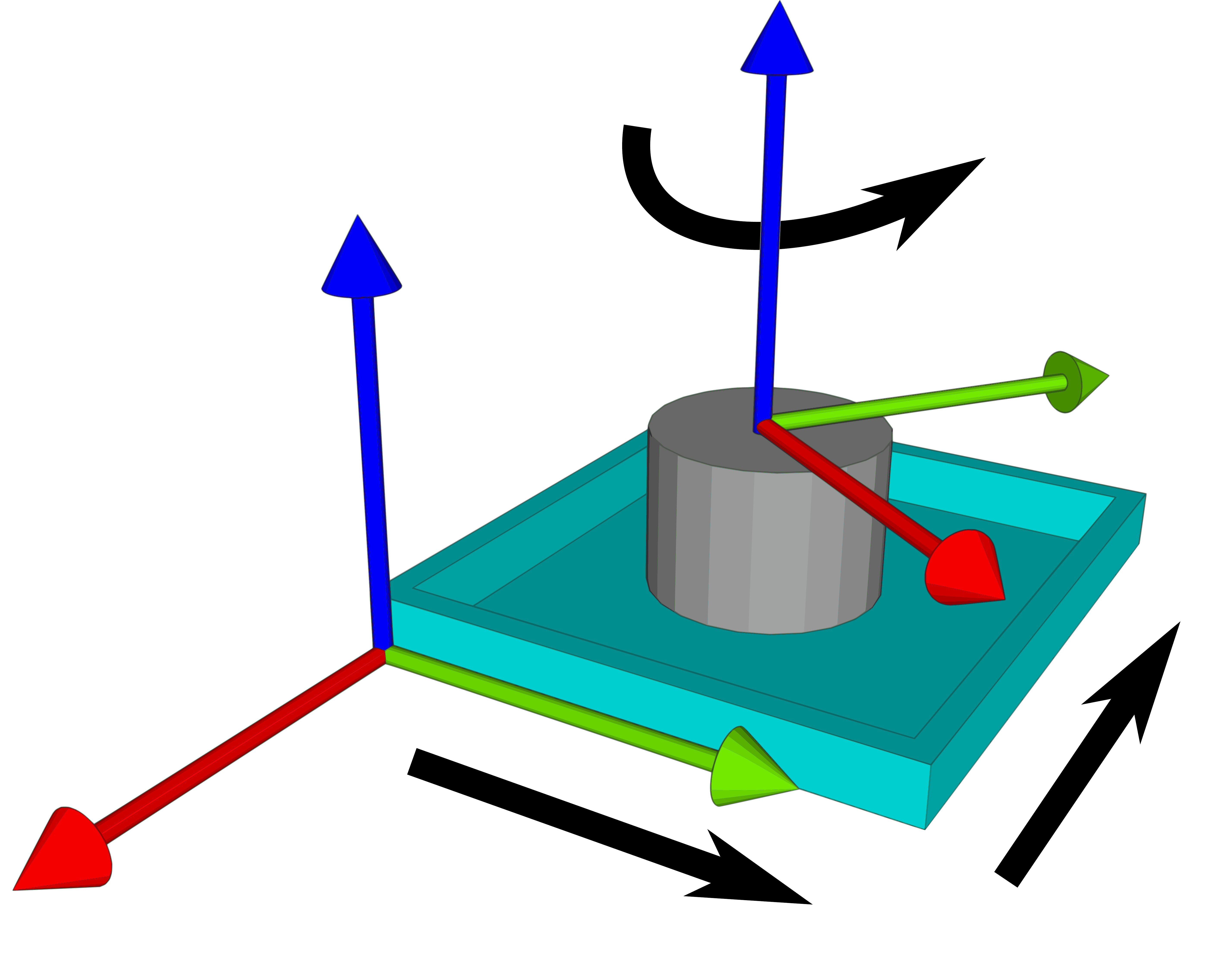}}%
  \end{picture}%
\endgroup%
} & \noindent \centering{}\resizebox{0.13\columnwidth}{!}{
\begingroup%
  \makeatletter%
  \providecommand\color[2][]{%
    \errmessage{(Inkscape) Color is used for the text in Inkscape, but the package 'color.sty' is not loaded}%
    \renewcommand\color[2][]{}%
  }%
  \providecommand\transparent[1]{%
    \errmessage{(Inkscape) Transparency is used (non-zero) for the text in Inkscape, but the package 'transparent.sty' is not loaded}%
    \renewcommand\transparent[1]{}%
  }%
  \providecommand\rotatebox[2]{#2}%
  \newcommand*\fsize{\dimexpr\f@size pt\relax}%
  \newcommand*\lineheight[1]{\fontsize{\fsize}{#1\fsize}\selectfont}%
  \ifx\svgwidth\undefined%
    \setlength{\unitlength}{1106.52428334bp}%
    \ifx\svgscale\undefined%
      \relax%
    \else%
      \setlength{\unitlength}{\unitlength * \real{\svgscale}}%
    \fi%
  \else%
    \setlength{\unitlength}{\svgwidth}%
  \fi%
  \global\let\svgwidth\undefined%
  \global\let\svgscale\undefined%
  \makeatother%
  \begin{picture}(1,1.76119611)%
    \lineheight{1}%
    \setlength\tabcolsep{0pt}%
    \put(0,0){\includegraphics[width=\unitlength,page=1]{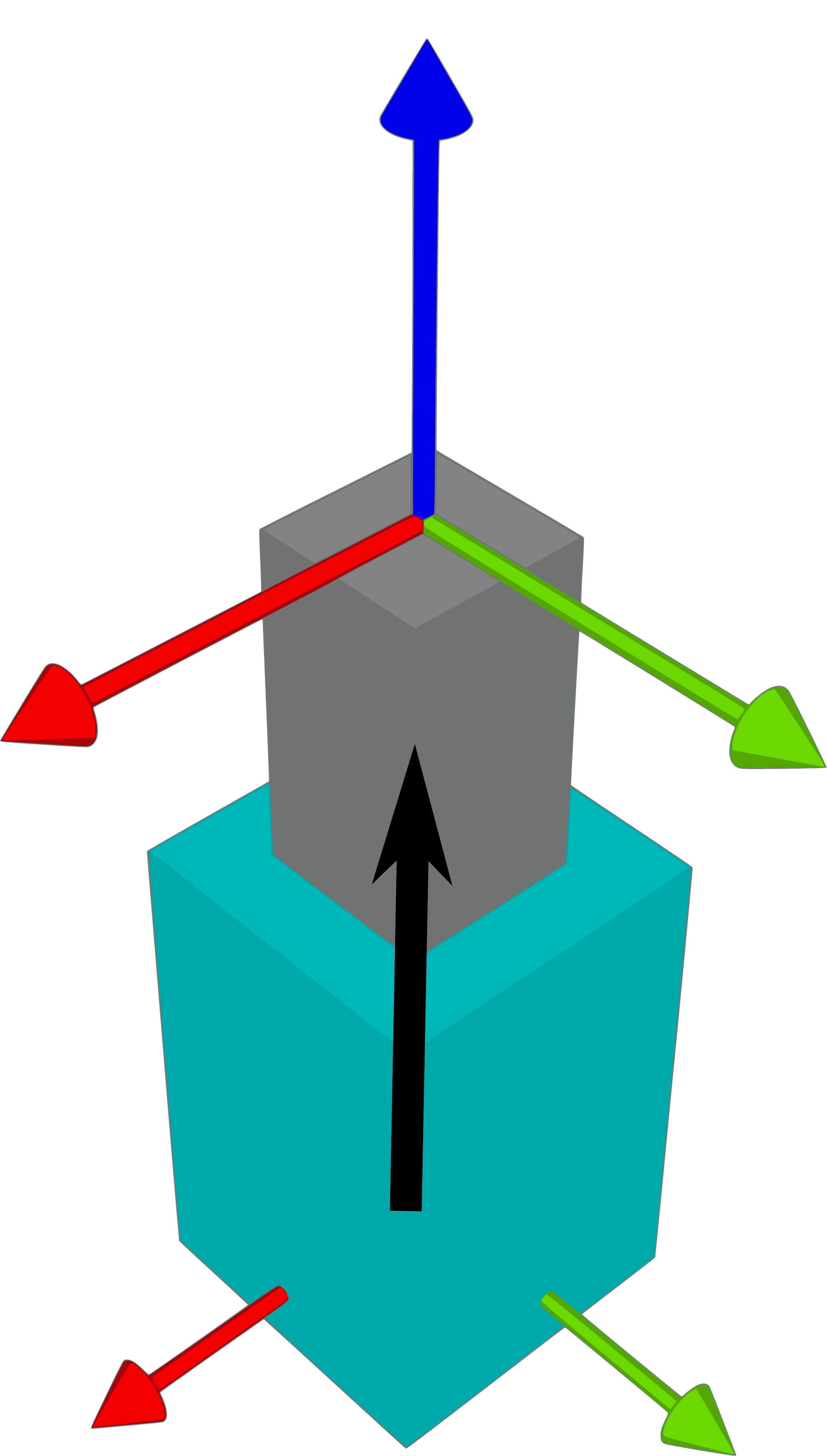}}%
    \put(0.61704539,0.39736169){\color[rgb]{0,0,0}\makebox(0,0)[lt]{\lineheight{1.25}\smash{\begin{tabular}[t]{l}$d$\end{tabular}}}}%
  \end{picture}%
\endgroup%
} & \resizebox{0.2\columnwidth}{!}{
\begingroup%
  \makeatletter%
  \providecommand\color[2][]{%
    \errmessage{(Inkscape) Color is used for the text in Inkscape, but the package 'color.sty' is not loaded}%
    \renewcommand\color[2][]{}%
  }%
  \providecommand\transparent[1]{%
    \errmessage{(Inkscape) Transparency is used (non-zero) for the text in Inkscape, but the package 'transparent.sty' is not loaded}%
    \renewcommand\transparent[1]{}%
  }%
  \providecommand\rotatebox[2]{#2}%
  \newcommand*\fsize{\dimexpr\f@size pt\relax}%
  \newcommand*\lineheight[1]{\fontsize{\fsize}{#1\fsize}\selectfont}%
  \ifx\svgwidth\undefined%
    \setlength{\unitlength}{1859.04865896bp}%
    \ifx\svgscale\undefined%
      \relax%
    \else%
      \setlength{\unitlength}{\unitlength * \real{\svgscale}}%
    \fi%
  \else%
    \setlength{\unitlength}{\svgwidth}%
  \fi%
  \global\let\svgwidth\undefined%
  \global\let\svgscale\undefined%
  \makeatother%
  \begin{picture}(1,1.14716445)%
    \lineheight{1}%
    \setlength\tabcolsep{0pt}%
    \put(0,0){\includegraphics[width=\unitlength,page=1]{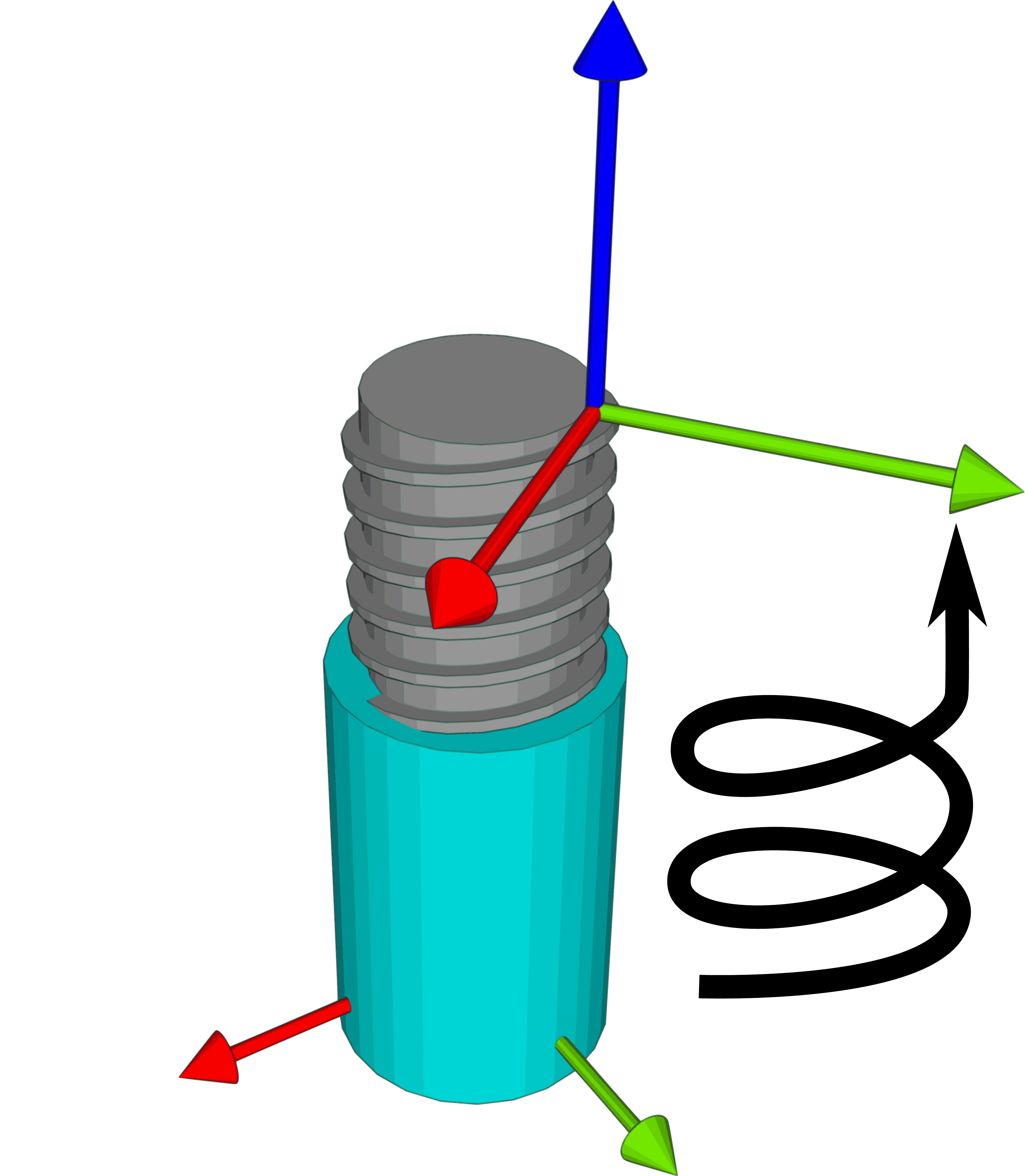}}%
  \end{picture}%
\endgroup%
}\tabularnewline
\hline 
\multirow{2}{3.3cm}{\noindent \centering{}$\dq{\xi}=\omega_{x}\imi+\omega_{y}\imj+\omega_{z}\imk+\dual\left(v_{x}\imi+v_{y}\imj+v_{z}\imk\right)$} & \multirow{2}{2.2cm}{\noindent \centering{}$\dq{\xi}=\omega\quat l$} & \multirow{2}{2.5cm}{\noindent \centering{}$\dq{\xi}=\omega_{x}\imi+\omega_{y}\imj\allowbreak+\omega_{z}\imk$} & \multirow{2}{3cm}{\noindent \centering{}$\dq{\xi}=\left(\omega+\dual v\right)\quat l$} & \multirow{2}{2.5cm}{\noindent \centering{}$\dq{\xi}=\omega\quat l+\allowbreak\dual\left(v_{x}\imi+v_{y}\imj\right)$} & \multirow{2}{2.5cm}{\noindent \centering{}$\dq{\xi}=\dual v\quat l$} & \multirow{2}{*}{$\dq{\xi}=\left(\omega+\dual h\omega\right)\quat l$}\tabularnewline
 &  &  &  &  &  & \tabularnewline
\hline 
\end{tabular}
\end{table*}

Moreover, $\dq{\xi}_{c_{1},j_{1}}^{c_{1}}=0$ because $\dot{\dq x}_{j_{1}}^{c_{1}}=0$.
Therefore,
\begin{align*}
\dq{\xi}_{0,c_{2}}^{c_{2}} & =\ad{\dq x_{c_{1}}^{c_{2}}}{\dq{\xi}_{0,c_{1}}^{c_{1}}}+\ad{\dq x_{j_{1}}^{c_{2}}}{\dq{\xi}_{j_{1},c_{2}}^{j_{1}}}.
\end{align*}
Furthermore, expanding $\ad{\dq x_{j_{1}}^{c_{2}}}{\dq{\xi}_{j_{1},c_{2}}^{j_{1}}}$,
we obtain
\[
\ad{\dq x_{j_{1}}^{c_{2}}}{\dq{\xi}_{j_{1},c_{2}}^{j_{1}}}=\quat{\omega}_{j_{1},c_{2}}^{c_{2}}+\dual\left(\quat v_{j_{1},c_{2}}^{c_{2}}+\quat{\omega}_{j_{1},c_{2}}^{c_{2}}\times\quat p_{j_{1},c_{2}}^{c_{2}}\right),
\]
where the linear velocity due to the application of an angular velocity
in a point displaced from the CoM (i.e., at $\frame{j_{1}}$) arises
algebraically. Fig.~\ref{fig:intuition_twist_transf} illustrates
this phenomenon when a purely rotational joint is used (i.e., $\dq{\xi}_{j_{i},c_{i+1}}^{j_{i}}=\quat{\omega}_{j_{i},c_{i+1}}^{j_{i}}=\omega_{i}\myvec n_{j_{i},c_{i+1}}^{j_{i}}$,
where $\myvec n_{j_{i},c_{i+1}}^{j_{i}}\in\mathbb{H}_{p}\cap\mathbb{S}^{3}$
is an arbitrary unit-norm rotation axis).

\begin{figure*}[t]
\def\svgwidth{2.0\columnwidth}
\begin{centering}
{\huge{}\scalebox{0.5}[0.5]{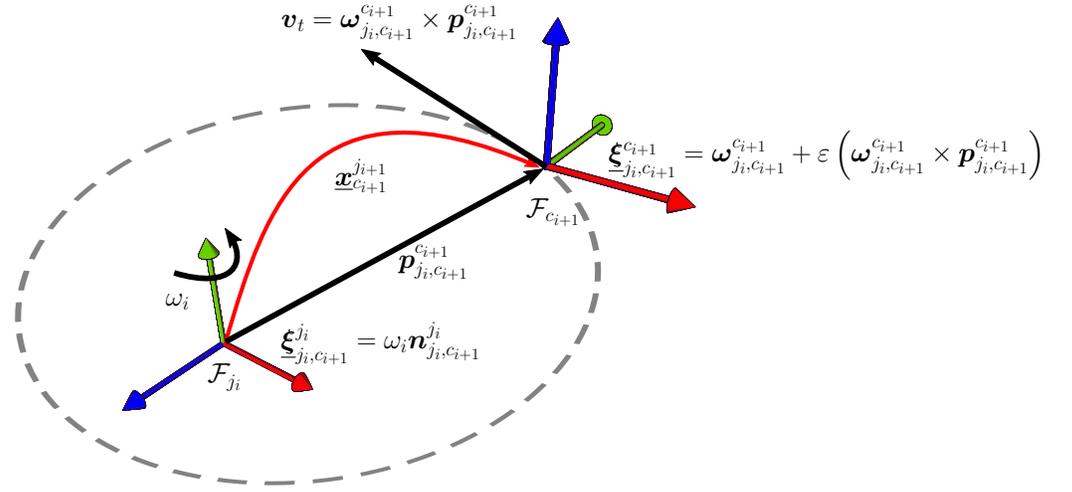}}{\huge\par}
\par\end{centering}
\caption{Twist $\protect\dq{\xi}_{j_{i},c_{i+1}}^{c_{i+1}}$generated due to
the application of an angular velocity $\omega_{i}$ around an arbitrary
axis of the reference frame $\protect\frame{j_{i}}$. The circular
trajectory that $\protect\frame{c_{i+1}}$ follows is represented
by the dashed gray line. The linear velocity due to the application
of $\omega_{i}$ appears algebraically through the adjoint transformation.
Thus, the tangential velocity of the reference frame $\protect\frame{c_{i+1}}$,
represented as a solid black arrow, is given by the dual part of the
twist $\protect\dq{\xi}_{j_{i},c_{i+1}}^{c_{i+1}}$. \label{fig:intuition_twist_transf}}
\end{figure*}

More generally, the twist in $\frame{c_{i}}$that provides the motion
of $\frame{c_{i}}$ with respect to $\frame 0$, which arises from
the movement of the first $i$ rigid bodies in the kinematic chain,
is given by
\begin{align}
\dq{\xi}_{0,c_{i}}^{c_{i}} & =\dq{\xi}_{0,j_{i-1}}^{c_{i}}+\dq{\xi}_{j_{i-1},c_{i}}^{c_{i}}\label{eq:twist_ci_0}
\end{align}
where $\dq{\xi}_{0,j_{i-1}}^{c_{i}}$ is the twist related to the
motion of the first $i-1$ rigid bodies and $\dq{\xi}_{j_{i-1},c_{i}}^{c_{i}}$
is the twist related to the motion of the $i$th rigid body. Also,
$\dq{\xi}_{0,0}^{a}=0$ for any $a$.

Analyzing \eqref{eq:twist_c1_0}, \eqref{eq:twist_c2_0}, and \eqref{eq:twist_ci_0}
we find, by induction, the recurrence relation for the total twist
of the $i$th CoM, which has the contribution of all rigid bodies
up to the $i$th rigid body, expressed in $\frame{c_{i}}$, as
\begin{multline*}
\dq{\xi}_{0,c_{i}}^{c_{i}}=\ad{\dq x_{c_{i-1}}^{c_{i}}}{\left(\dq{\xi}_{0,c_{i-1}}^{c_{i-1}}+\dq{\xi}_{c_{i-1},j_{i-1}}^{c_{i-1}}\right)}\\
+\ad{\dq x_{j_{i-1}}^{c_{i}}}{\dq{\xi}_{j_{i-1},c_{i}}^{j_{i-1}}},
\end{multline*}
where $c_{0}=j_{0}=0$ , and $\dq{\xi}_{c_{i-1},j_{i-1}}^{c_{i-1}}=0$
because $\dot{\dq x}_{j_{i-1}}^{c_{i-1}}=0$ for all $i$. Therefore,

\begin{equation}
\dq{\xi}_{0,c_{i}}^{c_{i}}=\ad{\dq x_{c_{i-1}}^{c_{i}}}{\dq{\xi}_{0,c_{i-1}}^{c_{i-1}}}+\ad{\dq x_{j_{i-1}}^{c_{i}}}{\dq{\xi}_{j_{i-1},c_{i}}^{j_{i-1}}.}\label{eq:twist_rec}
\end{equation}
Since the twist $\dq{\xi}_{j_{i-1},c_{i}}^{j_{i-1}}$ is generated
by the $i$th joint, its expression depends on which type the $i$th
joint is (\emph{see} Table~\ref{tab:joint_twists-2}). Similarly,
twist $\dq{\xi}_{0,c_{1}}^{c_{1}}$ depends on which type the mobile
base is (for holonomic mobile bases, for instance, it is equivalent
to the one given by a planar joint). The transformation $\dq x_{c_{i-1}}^{c_{i}}$
is calculated as $\dq x_{c_{i-1}}^{c_{i}}=\left(\dq x_{c_{i}}^{0}\right)^{*}\dq x_{c_{i-1}}^{0}$,
where $\dq x_{c_{i}}^{0}=\dq x_{c_{i-1}}^{0}\dq x_{j_{i-1}}^{c_{i-1}}\dq x_{c_{i}}^{j_{i-1}}$
with $\dq x_{0}^{0}=1$, the transformation $\dq x_{j_{i-1}}^{c_{i-1}}$
is constant, and $\dq x_{c_{i}}^{j_{i-1}}$ is a function of the parameters
of the $\left(i-1\right)$th joint (or the mobile base when $i=1$).
.

\subsubsection{Time Derivative of the Twists \label{subsec:time_derivative_twists}}

Taking the time derivative of \eqref{eq:twist_rec}, we use \eqref{eq:derivative_of_adjoint_transformation}
to obtain
\begin{multline*}
\dot{\dq{\xi}}_{0,c_{i}}^{c_{i}}=\ad{\dq x_{c_{i-1}}^{c_{i}}}{\dot{\dq{\xi}}_{0,c_{i-1}}^{c_{i-1}}}\\
+\dq{\xi}_{c_{i},c_{i-1}}^{c_{i}}\times\left(\ad{\dq x_{c_{i-1}}^{c_{i}}}{\dq{\xi}_{0,c_{i-1}}^{c_{i-1}}}\right)+\ad{\dq x_{j_{i-1}}^{c_{i}}}{\dot{\dq{\xi}}_{j_{i-1},c_{i}}^{j_{i-1}}}\\
+\dq{\xi}_{c_{i},j_{i-1}}^{c_{i}}\times\left(\ad{\dq x_{j_{i-1}}^{c_{i}}}{\dq{\xi}_{j_{i-1},c_{i}}^{j_{i-1}}}\right).
\end{multline*}
Since $\dq{\xi}_{j_{i-1},c_{i}}^{c_{i}}=-\dq{\xi}_{c_{i},j_{i-1}}^{c_{i}}$
then
\[
\dq{\xi}_{c_{i},j_{i-1}}^{c_{i}}\times\left(\ad{\dq x_{j_{i-1}}^{c_{i}}}{\dq{\xi}_{j_{i-1},c_{i}}^{j_{i-1}}}\right)=-\dq{\xi}_{c_{i},j_{i-1}}^{c_{i}}\times\dq{\xi}_{c_{i},j_{i-1}}^{c_{i}}=0.
\]
Therefore,

\begin{multline}
\dot{\dq{\xi}}_{0,c_{i}}^{c_{i}}=\ad{\dq x_{c_{i-1}}^{c_{i}}}{\dot{\dq{\xi}}_{0,c_{i-1}}^{c_{i-1}}}+\ad{\dq x_{j_{i-1}}^{c_{i}}}{\dot{\dq{\xi}}_{j_{i-1},c_{i}}^{j_{i-1}}}\\
+\dq{\xi}_{c_{i},c_{i-1}}^{c_{i}}\times\left[\ad{\dq x_{c_{i-1}}^{c_{i}}}{\dq{\xi}_{0,c_{i-1}}^{c_{i-1}}}\right],\label{eq:twist_dot_rec}
\end{multline}
where $\dot{\dq{\xi}}_{0,c_{0}}^{c_{0}}\triangleq0$. Also, since
$\dq{\xi}_{c_{i},c_{i-1}}^{j_{i-1}}=\dq{\xi}_{c_{i},j_{i-1}}^{j_{i-1}}+\dq{\xi}_{j_{i-1},c_{i-1}}^{j_{i-1}}$
and $\dq{\xi}_{j_{i-1},c_{i-1}}^{j_{i-1}}=0$, then
\begin{align}
\dq{\xi}_{c_{i},c_{i-1}}^{c_{i}} & =\ad{\dq x_{j_{i-1}}^{c_{i}}}{\dq{\xi}_{c_{i},j_{i-1}}^{j_{i-1}}}=-\ad{\dq x_{j_{i-1}}^{c_{i}}}{\dq{\xi}_{j_{i-1},c_{i}}^{j_{i-1}}}.\label{eq:twists_between_adjacent_CoM}
\end{align}

As shown in Section~\ref{subsec:Twists}, the twist $\dq{\xi}_{j_{i-1},c_{i}}^{j_{i-1}}$
depends on the type of the $i$th joint and, therefore, so does the
term $\dot{\dq{\xi}}_{j_{i-1},c_{i}}^{j_{i-1}}$. For instance, if
the $i$th joint is revolute, then $\dot{\dq{\xi}}_{j_{i-1},c_{i}}^{j_{i-1}}=\dot{\omega}_{i}\quat l_{j_{i}}^{j_{i}}$.
If it is prismatic, then $\dot{\dq{\xi}}_{j_{i-1},c_{i}}^{j_{i-1}}=\dual\dot{v}_{i}\quat l_{j_{i}}^{j_{i}}$.
Analogously, if it is helical, then $\dot{\dq{\xi}}_{j_{i-1},c_{i}}^{j_{i-1}}=\left(\dot{\omega}_{i}+\dual h\dot{\omega}_{i}\right)\quat l_{j_{i}}^{j_{i}}$,
etc. The same reasoning applies to the twist $\dot{\dq{\xi}}_{0,c_{1}}^{c_{1}}$
of the mobile base.

\selectlanguage{english}%

\selectlanguage{american}%
\begin{remark}

Although \eqref{eq:twist_dot_rec} is written in recursive form, we
can always write twists as in \eqref{eq:twist-inertial-frame} and
\eqref{eq:twist-body-frame}. Therefore, as $\dq{\xi}_{0,c_{i}}^{c_{i}}=\ad{\dq x_{0}^{c_{i}}}{\dq{\xi}_{0,c_{i}}^{0}}$,
with $\dq{\xi}_{0,c_{i}}^{0}=\quat{\omega}_{0,c_{i}}^{0}+\dual(\dot{\quat p}_{0,c_{i}}^{0}+\quat p_{0,c_{i}}^{0}\times\quat{\omega}_{0,c_{i}}^{0})$,
we use \eqref{eq:derivative_of_adjoint_transformation} to obtain
\begin{align}
\dot{\dq{\xi}}_{0,c_{i}}^{c_{i}} & {=}\ad{\dq x_{0}^{c_{i}}}{\dot{\dq{\xi}}_{0,c_{i}}^{0}{=}\dot{\quat{\omega}}_{0,c_{i}}^{c_{i}}{+}\dual\left(\ddot{\quat p}_{0,c_{i}}^{c_{i}}{+}\dot{\quat p}_{0,c_{i}}^{c_{i}}{\times}\quat{\omega}_{0,c_{i}}^{c_{i}}\right)}\label{eq:twist_derivative_explicit_form}
\end{align}
 because $\dq{\xi}_{c_{i},0}^{c_{i}}\times\ad{\dq x_{0}^{c_{i}}}{\dq{\xi}_{0,c_{i}}^{0}}=-\dq{\xi}_{0,c_{i}}^{c_{i}}\times\dq{\xi}_{0,c_{i}}^{c_{i}}=0$.
Since $\getd{\dot{\dq{\xi}}_{0,c_{i}}^{c_{i}}}=\ddot{\quat p}_{0,c_{i}}^{c_{i}}+\dot{\quat p}_{0,c_{i}}^{c_{i}}\times\quat{\omega}_{0,c_{i}}^{c_{i}}$
then
\begin{gather}
\ddot{\quat p}_{0,c_{i}}^{c_{i}}=\getd{\dot{\dq{\xi}}_{0,c_{i}}^{c_{i}}}-\getd{\dq{\xi}_{0,c_{i}}^{c_{i}}}\times\getp{\dq{\xi}_{0,c_{i}}^{c_{i}}}.\label{eq:linear-acceleration}
\end{gather}

\end{remark}

\subsection{Backward Recursion\label{subsec:dqNE_backward_recursion}}

The second process of the iterative algorithm consists in sweeping
the serial robot from the last to the first rigid body to calculate
the wrenches applied at each one of them. For the robotic arm, we
are interested in the wrenches at each joint, whereas for the mobile
base we want to find the wrench at its CoM. To that aim, we use the
twists obtained in Section~\ref{subsec:Forward-Recursion} and their
time derivatives.

Before obtaining the general expression for the backward recursion,
let us consider the mobile manipulator shown in Fig.~\ref{fig:mobile_nDOF_robot}.
The wrench at the CoM of the $n_{\ell}$th link (i.e., the $n$th
CoM in the kinematic chain, in which $n=n_{\ell}+1$), expressed in
$\frame{c_{n}}$, is given by the pure dual quaternion

\begin{equation}
\dq{\zeta}_{0,c_{n}}^{c_{n}}=\dq{\varsigma}_{0,c_{n}}^{c_{n}}-m_{n}\quat g^{c_{n}},\label{eq:wrench_cn_in_cn}
\end{equation}
where $m_{n}\quat g^{c_{n}}$ is the gravitational component, with
$\quat g^{c_{n}}\in\mathbb{H}_{p}$ being the gravity vector expressed
in $\frame{c_{n}}$, and $\dq{\varsigma}_{0,c_{n}}^{c_{n}}=\quat f_{0,c_{n}}^{c_{n}}+\dual\quat{\tau}_{0,c_{n}}^{c_{n}}$,
in which $\quat f_{0,c_{n}}^{c_{n}}=f_{x}\imi+f_{y}\imj+f_{z}\imk$
is the force at the CoM of the $n$th rigid body (i.e., the $n_{\ell}$th
link), given by Newton's second law $\quat f_{0,c_{n}}^{c_{n}}=m_{n}\ddot{\quat p}_{0,c_{n}}^{c_{n}}.$

Therefore, we use \eqref{eq:linear-acceleration} to obtain

\begin{align}
\quat f_{0,c_{n}}^{c_{n}} & =m_{n}\left(\getd{\dot{\dq{\xi}}_{0,c_{n}}^{c_{n}}}+\getp{\dq{\xi}_{0,c_{n}}^{c_{n}}}\times\getd{\dq{\xi}_{0,c_{n}}^{c_{n}}}\right).\label{eq:newton_sec_law}
\end{align}
Furthermore, $\quat{\tau}_{0,c_{n}}^{c_{n}}$ is the torque about
the $n$th rigid body's CoM due to the change of its angular momentum,
given by the Euler's rotation equation

\begin{multline}
\quat{\tau}_{0,c_{n}}^{c_{n}}=\mathcal{L}_{3}\left(\quat{\mathbb{I}}_{n}^{c_{n}}\right)\getp{\dot{\dq{\xi}}_{0,c_{n}}^{c_{n}}}\\
+\getp{\dq{\xi}_{0,c_{n}}^{c_{n}}}\times\left(\mathcal{L}_{3}\left(\quat{\mathbb{I}}_{n}^{c_{n}}\right)\getp{\dq{\xi}_{0,c_{n}}^{c_{n}}}\right),\label{eq:euler_eq}
\end{multline}
where $\mathcal{L}_{3}$ is given by \eqref{eq:operator_l} and $\quat{\mathbb{I}}_{n}^{c_{n}}$
is the quaternionic inertia tensor of the $n$th rigid body, expressed
at its CoM, given by \eqref{eq:quat_inertia_tensor}. Because \eqref{eq:euler_eq}
is calculated with respect to the CoM, the gravity acceleration does
not contribute to the torque.

Using the adjoint transformation as in \eqref{eq:adj_transf} in \eqref{eq:wrench_cn_in_cn},
the wrench at the $n_{\ell}$th joint, resulting from the wrench at
the CoM of the $n$th rigid body, is given by
\begin{align}
\dq{\zeta}_{0,j_{n}}^{j_{n-1}} & =\ad{\dq x_{c_{n}}^{j_{n-1}}}{\dq{\zeta}_{0,c_{n}}^{c_{n}}}.\label{eq:wrench_cn_in_jnMinus1}
\end{align}

The resultant wrench at the $\left(n-1\right)$th rigid body (i.e.,
at $\frame{c_{n-1}}$) also includes the effects of the wrench from
the $n$th rigid body as they are rigidly attached to each other.
Therefore, the resultant wrench at the $\left(n_{\ell}-1\right)$th
joint (i.e., at $\frame{c_{n-2}}$) is given by
\begin{alignat}{1}
\dq{\zeta}_{0,j_{n-1}}^{j_{n-2}} & =\ad{\dq x_{c_{n-1}}^{j_{n-2}}}{\dq{\zeta}_{0,c_{n-1}}^{c_{n-1}}}+\ad{\dq x_{j_{n-1}}^{j_{n-2}}}{\dq{\zeta}_{0,j_{n}}^{j_{n-1}}},\label{eq:wrench_cnMinus1_in_jnMinus2}
\end{alignat}
where $\dq{\zeta}_{0,c_{n-1}}^{c_{n-1}}=\dq{\varsigma}_{0,c_{n-1}}^{c_{n-1}}-m_{n-1}\quat g^{c_{n-1}}$,
with $\dq{\varsigma}_{0,c_{n-1}}^{c_{n-1}}=\quat f_{0,c_{n-1}}^{c_{n-1}}+\dual\quat{\tau}_{0,c_{n-1}}^{c_{n-1}}$,
is the wrench at the CoM of the $\left(n-1\right)$th rigid body expressed
in $\frame{c_{n-1}}$.

Thus, analyzing \eqref{eq:wrench_cn_in_cn}, \eqref{eq:wrench_cn_in_jnMinus1},
and \eqref{eq:wrench_cnMinus1_in_jnMinus2}, we find the backward
recurrence relation for the total wrench at the $i$th rigid body,
which includes the contribution of all wrenches starting at the CoM
of the $i$th rigid body up to the wrench at the CoM of the last one,
expressed in $\frame{j_{i-1}}$, as
\begin{align}
\dq{\zeta}_{0,j_{i}}^{j_{i-1}} & =\ad{\dq x_{c_{i}}^{j_{i-1}}}{\dq{\zeta}_{0,c_{i}}^{c_{i}}}+\ad{\dq x_{j_{i}}^{j_{i-1}}}{\dq{\zeta}_{0,j_{i+1}}^{j_{i}}},\label{eq:wrench_rec}
\end{align}
with $i\in\{1,\ldots,n\}$ and $c_{0}=j_{0}=0$, where $\dq{\zeta}_{0,j_{n_{\ell}+2}}^{j_{n_{\ell}+1}}=\dq{\zeta}_{0,j_{n+1}}^{j_{n}}=0$
(recall that $n=n_{\ell}+1$) and $\dq{\zeta}_{0,c_{i}}^{c_{i}}=\dq{\varsigma}_{0,c_{i}}^{c_{i}}-m_{i}\quat g^{c_{i}}$,
with $\dq{\varsigma}_{0,c_{i}}^{c_{i}}=\quat f_{0,c_{i}}^{c_{i}}+\dual\quat{\tau}_{0,c_{i}}^{c_{i}}$,
is the wrench at the $i$th CoM,\footnote{If an \emph{external }wrench is applied at the end-effector, then
$\dq{\zeta}_{0,j_{n+1}}^{j_{n}}\neq0$.} $\quat f_{0,c_{i}}^{c_{i}}=m_{i}\ddot{\quat p}_{0,c_{i}}^{c_{i}}$,
and $\quat{\tau}_{0,c_{i}}^{c_{i}}=\mathcal{L}_{3}\left(\quat{\mathbb{I}}_{i}^{c_{i}}\right)\getp{\dot{\dq{\xi}}_{0,c_{i}}^{c_{i}}}+\getp{\dq{\xi}_{0,c_{i}}^{c_{i}}}\times\left(\mathcal{L}_{3}\left(\quat{\mathbb{I}}_{i}^{c_{i}}\right)\getp{\dq{\xi}_{0,c_{i}}^{c_{i}}}\right)$.
For the mobile base (i.e., the first CoM in the serial kinematic chain;
hence, $i=1$), notice that \eqref{eq:wrench_rec} results in $\dq{\zeta}_{0,j_{1}}^{j_{0}}=\ad{\dq x_{c_{1}}^{j_{0}}}{\dq{\zeta}_{0,c_{1}}^{c_{1}}}+\ad{\dq x_{j_{1}}^{j_{0}}}{\dq{\zeta}_{0,j_{2}}^{j_{1}}}$,
which implies $\dq{\zeta}_{0,j_{1}}^{0}=\ad{\dq x_{c_{1}}^{0}}{\dq{\zeta}_{0,c_{1}}^{c_{1}}}+\ad{\dq x_{j_{1}}^{0}}{\dq{\zeta}_{0,j_{2}}^{j_{1}}}$.

Moreover, the transformation $\dq x_{c_{i}}^{j_{i-1}}$ is a function
of the joint or mobile base coordinates. For example, the transformation
of the mobile base, $\dq x_{c_{1}}^{j_{0}}=\dq x_{c_{1}}^{0},$ depends
on the its coordinates and rotation angle, namely $(x_{\mathrel{\mathrm{base}}},y_{\mathrel{\mathrm{base}}},\phi_{\mathrel{\mathrm{base}}})$,
therefore $\dq x_{c_{1}}^{0}\triangleq\dq x_{c_{1}}^{0}(x_{\mathrel{\mathrm{base}}},y_{\mathrel{\mathrm{base}}},\phi_{\mathrel{\mathrm{base}}})$
and $q_{1}=x_{\mathrel{\mathrm{base}}},q_{2}=y_{\mathrm{\mathrel{base}}},q_{3}=\phi_{\mathrm{\mathrel{base}}}$.
For manipulators with revolute, prismatic, or helicoidal joints, the
transformation $\dq x_{c_{i}}^{j_{i-1}}$, with $i\geq2$, is a function
of just one parameter, that is $\dq x_{c_{i}}^{j_{i-1}}\triangleq\dq x_{c_{i}}^{j_{i-1}}(q_{j_{i-1}})$.
In that case, $q_{j_{i-1}}=q_{i+2}$. Analogously, spherical and planar
joints depend on three parameters whereas helicoidal joints depend
on six parameters. Therefore, one must be careful when defining the
index for each parameter within the configuration vector $\myvec q$.

\subsubsection{Particular cases: prismatic and revolute joints}

In the case of manipulator robots with revolute and/or prismatic joints,
which are the most common ones, the wrenches given by \eqref{eq:wrench_rec}
must be projected onto the joints motion axes through
\begin{equation}
\dotproduct{\dq{\zeta}_{0,j_{i}}^{j_{i-1}},\quat l_{j_{i-1}}}=f_{\quat l_{i}}+\dual\tau_{\quat l_{i}},\label{eq:wrench_projection_to_motion_axis}
\end{equation}
where $f_{\quat l_{i}},\tau_{\quat l_{i}}\in\mathbb{R}$ and $\dotproduct{\dq{\zeta}_{0,j_{i}}^{j_{i-1}},\quat l_{j_{i-1}}}$
is the inner product between the wrench $\dq{\zeta}_{0,j_{i}}^{j_{i-1}}=\quat f_{0,j_{i}}^{j_{i-1}}+\dual\quat{\tau}_{0,j_{i}}^{j_{i-1}}$
and the motion axis $\quat l_{j_{i-1}}\in\mathbb{H}_{p}\cap\mathbb{S}^{3}$
of the $i$th joint, given by \cite{Adorno2017}
\begin{align*}
\dotproduct{\dq{\zeta}_{0,j_{i}}^{j_{i-1}},\quat l_{j_{i-1}}} & =-\frac{\left(\dq{\zeta}_{0,j_{i}}^{j_{i-1}}\quat l_{j_{i-1}}+\quat l_{j_{i-1}}\dq{\zeta}_{0,j_{i}}^{j_{i-1}}\right)}{2}\\
 & =\dotproduct{\quat f_{0,j_{i}}^{j_{i-1}},\quat l_{j_{i-1}}}+\dual\dotproduct{\quat{\tau}_{0,j_{i}}^{j_{i-1}},\quat l_{j_{i-1}}}=f_{\quat l_{i}}+\dual\tau_{\quat l_{i}}.
\end{align*}
Therefore, if the $i$th joint is revolute, then the corresponding
torque is given by $\tau_{\quat l_{i}}=\getd{\dotproduct{\dq{\zeta}_{0,j_{i}}^{j_{i-1}},\quat l_{j_{i-1}}}}$.
If it is prismatic, then the corresponding force along along the axis
$\quat l_{j_{i-1}}$ is given by $f_{\quat l_{i}}=\getp{\dotproduct{\dq{\zeta}_{0,j_{i}}^{j_{i-1}},\quat l_{j_{i-1}}}}$.

\subsubsection{Particular case: planar joints/holonomic base}

For robots with holonomic mobile bases and/or planar joints (which
are kinematically equivalent), we must project the wrenches onto all
the three axes of motion of the joint/mobile base.\footnote{Notice that this procedure applies for all joints with more than one
axis of movement (e.g., spherical joints, planar joints, etc.).} That is, the corresponding forces along the $x$-axis and $y$-axis
of the joint/mobile base are given by 
\begin{align*}
f_{\quat l_{i_{x}}} & =\getp{\dotproduct{\dq{\zeta}_{0,j_{i}}^{j_{i-1}},\ad{\quat r_{0}^{j_{i-1}}}{\imi}}},\\
f_{\quat l_{i_{y}}} & =\getp{\dotproduct{\dq{\zeta}_{0,j_{i}}^{j_{i-1}},\ad{\quat r_{0}^{j_{i-1}}}{\imj}}},
\end{align*}
whereas the corresponding torque about the $z$-axis of the joint/mobile
base is given by 
\begin{align*}
\tau_{\quat l_{z}} & =\getd{\dotproduct{\dq{\zeta}_{0,j_{i}}^{j_{i-1}},\ad{\quat r_{0}^{j_{i-1}}}{\imk}}}.
\end{align*}

\selectlanguage{english}%

\section{Gauss's Principle of Least Constraint\label{sec:Gauss-Principle}}

The GPLC \cite{Kalaba1993} is a differential variational principle,
equivalent to the D'Alembert one, that is based on the variation of
the acceleration. For a system composed of $n$ bodies, it can be
stated as the least-squares minimization problem
\begin{gather}
\minim{}{\sum_{i=1}^{n}\frac{1}{2}\left(\myvec a_{c_{i}}-\bar{\myvec a}_{c_{i}}\right)^{T}\mymatrix{\Psi}_{c_{i}}\left(\myvec a_{c_{i}}-\bar{\myvec a}_{c_{i}}\right)},\label{eq:GaussPrinciple}
\end{gather}
where $\myvec a_{c_{i}}$ and $\bar{\myvec a}_{c_{i}}$ are the accelerations
of the center of mass of the $i$th rigid body under constraints and
without constraints, respectively. Furthermore, $\mymatrix{\Psi}_{c_{i}}\triangleq\mymatrix{\Psi}_{c_{i}}\left(\bar{\mymatrix{\mathbb{I}}}_{i},m_{i}\right)$
encapsulates the inertial parameters of the $i$th rigid body, such
as the inertia tensor $\bar{\mymatrix{\mathbb{I}}}_{i}\in\mathbb{R}^{3\times3}$
and the mass $m_{i}$.

This principle has been used in robotics to describe the dynamics
of robot manipulators \cite{Bruyninckx2000a} and rigid body simulations
\cite{Redon2002}. Wieber \cite{Wieber2006} uses the GPLC to derive
the analytic expression of the Lagrangian dynamics of a humanoid robot.
Bouyarmane and Kheddar \cite{Bouyarmane2012} extend Wieber's work
by handling arbitrary parameterization of free-floating-base mechanisms.
This allows using rotation matrices or unit quaternions to represent
the free-floating-base orientations. In this section, we rewrite the
GPLC for articulated bodies, similar to Wieber's formulation \cite{Wieber2006},
but using dual quaternion algebra. This allows a more compact and
unified representation than the one by Bouyarmane and Kheddar \cite{Bouyarmane2012}.

First, we rewrite the constrained accelerations \eqref{eq:twist_dot_rec}
as a linear function of the vector of joints velocities and joints
accelerations. This allows solving \eqref{eq:GaussPrinciple} for
the joints accelerations and, therefore, additional constraints can
be included in the optimization formulation. We then define constraints
to model nonholonomic mobile manipulators. Different from \cite{Bouyarmane2012},
we do not use Lagrange multipliers. Instead, we apply Udwadia-Kalaba's
fundamental equation \cite{FirdausE.UdwadiaandRobertE.Kalaba1992},
which is a simpler method for solving quadratic optimization problems
such as \eqref{eq:GaussPrinciple}.

\subsection{Constrained acceleration $\protect\myvec a_{c_{i}}$\label{subsec:Constrained-acceleration}}

Consider the robotic system in Fig.~\ref{fig:mobile_nDOF_robot}.
The robot is composed of rigid bodies that are constrained\footnote{In this case, there are both holonomic and nonholonomic constraints.
The former are constraints between adjacent links in the kinematic
chain. The latter is the constraint of the mobile base.} to one another by joints. To express the twist $\dq{\xi}_{0,c_{i}}^{c_{i}}$
of the $i$th center of mass explicitly as a linear combination between
its Jacobian $\mymatrix J_{\dq{\xi}_{0,c_{i}}^{c_{i}}}$ and the vector
of joints velocities $\dot{\myvec q}\in\mathbb{R}^{n}$, we use the
operators $\vector_{8}:\mathcal{H}\to\mathbb{R}^{8}$, which maps
the coefficients of a dual quaternion into an eight-dimensional vector,\footnote{Given $\dq h=h_{1}+\imi h_{2}+\imj h_{3}+\imk h_{4}+\dual\left(h_{5}+\imi h_{6}+\imj h_{7}+\imk h_{8}\right)$,
$\vector_{8}\dq h=\begin{bmatrix}h_{1} & \cdots & h_{8}\end{bmatrix}^{T}$.} and $\hami +_{8}:\mathcal{H}\to\mathbb{R}^{8\times8}$, such that
$\vector_{8}\left(\dq h_{1}\dq h_{2}\right)=\hami +_{8}\left(\dq h_{1}\right)\vector_{8}\dq h_{2}$
\cite{Adorno2011}. Therefore, from \eqref{eq:x_dot} we obtain $\dq{\xi}_{0,c_{i}}^{c_{i}}=2\left(\dq x_{c_{i}}^{0}\right)^{*}\dot{\dq x}_{c_{i}}^{0}$,
which implies $\vector_{8}\dq{\xi}_{0,c_{i}}^{c_{i}}=2\hami +_{8}\left(\dq x_{0}^{c_{i}}\right)\vector_{8}\dot{\dq x}_{c_{i}}^{0}$.

Because $\dq{\xi}_{0,c_{i}}^{c_{i}}\in\mathcal{H}_{p}$, the first
and fifth elements of $\vector_{8}\dq{\xi}_{0,c_{i}}^{c_{i}}$ equal
zero, thus we also use the operator $\vector_{6}:\mathcal{H}_{p}\to\mathbb{R}^{6}$
such that $\vector_{6}\dq{\xi}_{0,c_{i}}^{c_{i}}\triangleq\bar{\mymatrix I}\vector_{8}\dq{\xi}_{0,c_{i}}^{c_{i}}$,
where
\[
\bar{\mymatrix I}\triangleq\begin{bmatrix}\myvec 0_{3\times1} & \mymatrix I_{3} & \myvec 0_{3\times1} & \myvec 0_{3\times3}\\
\myvec 0_{3\times1} & \myvec 0_{3\times3} & \myvec 0_{3\times1} & \mymatrix I_{3}
\end{bmatrix},
\]
with $\mymatrix I_{3}\in\mathbb{R}^{3\times3}$ being the identity
matrix and $\mymatrix 0_{m\times n}\in\mathbb{R}^{m\times n}$ being
a matrix of zeros. Moreover, $\vector_{8}\dot{\dq x}_{c_{i}}^{0}=\mymatrix J_{\dq x_{c_{i}}^{0}}\dot{\myvec q}_{i}$,
with $\dot{\myvec q}_{i}=\begin{bmatrix}\dot{q}_{1} & \cdots & \dot{q}_{i}\end{bmatrix}^{T}$,
and $\mymatrix J_{\dq x_{c_{i}}^{0}}\in\mathbb{R}^{8\times i}$ is
the Jacobian matrix that is obtained algebraically \cite{Adorno2011}.
Hence,
\begin{align}
\myvec{\nu}_{c_{i}}\triangleq\vector_{6}\dq{\xi}_{0,c_{i}}^{c_{i}} & =\underset{\mymatrix J_{\dq{\xi}_{0,c_{i}}^{c_{i}}}}{\underbrace{\left[\begin{array}{cc}
\bar{\mymatrix J}_{\dq{\xi}_{0,c_{i}}^{c_{i}}} & \myvec 0_{6\times(n-i)}\end{array}\right]}}\dot{\myvec q},\label{eq:ConstrainedVel}
\end{align}
where $\bar{\mymatrix J}_{\dq{\xi}_{0,c_{i}}^{c_{i}}}=2\bar{\mymatrix I}\hami +_{8}\left(\dq x_{0}^{c_{i}}\right)\mymatrix J_{\dq x_{c_{i}}^{0}}\in\mathbb{R}^{6\times i}$.

Finally, the constrained acceleration of the $i$th center of mass
is given by
\begin{align}
\myvec a_{c_{i}} & \triangleq\vector_{6}\dot{\dq{\xi}}_{0,c_{i}}^{c_{i}}=\mymatrix J_{\dq{\xi}_{0,c_{i}}^{c_{i}}}\ddot{\myvec q}+\dot{\mymatrix J}_{\dq{\xi}_{0,c_{i}}^{c_{i}}}\dot{\myvec q}.\label{eq:ConstrainedAcceleration}
\end{align}

We recall that \eqref{eq:ConstrainedAcceleration} is equivalent to
\eqref{eq:twist_dot_rec} as the Newton-Euler formalism implicitly
considers the linkage constraints of the bodies.

\subsection{Unconstrained acceleration $\bar{\protect\myvec a}_{c_{i}}$}

Consider $\dq x_{c_{i}}^{0}=\quat r_{c_{i}}^{0}+(1/2)\dual\quat p_{0,c_{i}}^{0}\quat r_{c_{i}}^{0}$,
which represents the rigid motion from $\frame 0$ to $\frame{c_{i}}$,
and the twist $\overline{\dq{\xi}}_{0,c_{i}}^{c_{i}}$ at $\frame{c_{i}}$of
the $i$th body under no constraints. From \eqref{eq:twist_derivative_explicit_form},
the unconstrained acceleration is given explicitly as
\begin{align}
\bar{\myvec a}_{c_{i}} & \triangleq\vector_{6}\dot{\overline{\dq{\xi}}}_{0,c_{i}}^{c_{i}}=\left[\begin{array}{c}
\vector_{3}\dot{\quat{\omega}}_{0,c_{i}}^{c_{i}}\\
\vector_{3}\left(\ddot{\quat p}_{0,c_{i}}^{c_{i}}+\dot{\quat p}_{0,c_{i}}^{c_{i}}\times\quat{\omega}_{0,c_{i}}^{c_{i}}\right)
\end{array}\right],\label{eq.UnconstrainedAcceleration}
\end{align}
where $\vector_{3}:\mathbb{H}_{p}\to\mathbb{R}^{3}$ such that $\vector_{3}(a\imi+b\imj+c\imk)=\begin{bmatrix}a & b & c\end{bmatrix}^{T}$.
Whereas \eqref{eq:ConstrainedAcceleration} depends on $\myvec q$,
$\dot{\myvec q}$, and $\ddot{\myvec q}$, Eq.~\eqref{eq.UnconstrainedAcceleration}
does not because it is unconstrained.

\subsection{Euler Lagrange equations\label{subsec:Euler-Lagrange-equations}}

Let $\mathcal{G}\left(\myvec q,\dot{\myvec q},\ddot{\myvec q}\right)=\sum_{i=1}^{n}\frac{1}{2}\left(\myvec a_{c_{i}}-\bar{\myvec a}_{c_{i}}\right)^{T}\mymatrix{\Psi}_{c_{i}}\left(\myvec a_{c_{i}}-\bar{\myvec a}_{c_{i}}\right)$,
in which $\myvec a_{c_{i}}$ and $\bar{\myvec a}_{c_{i}}$ are given
by \eqref{eq:ConstrainedAcceleration} and \eqref{eq.UnconstrainedAcceleration},
where $\mymatrix{\Psi}_{c_{i}}\triangleq\mathrm{blkdiag}\left(\bar{\mymatrix{\mathbb{I}}}_{i}^{c_{i}},m_{i}\mymatrix I_{3}\right)$.

Expanding $\mathcal{G}\left(\myvec q,\dot{\myvec q},\ddot{\myvec q}\right)$,
we obtain
\begin{align}
\mathcal{G}\left(\myvec q,\dot{\myvec q},\ddot{\myvec q}\right) & =\underset{i=1}{\sum^{n}}\left(\mathcal{G}_{a_{i}}\left(\myvec q,\dot{\myvec q},\ddot{\myvec q}\right)+\mathcal{G}_{b_{i}}\left(\myvec q,\dot{\myvec q}\right)\right),\label{eq:ObjectiveFunctionRo}
\end{align}
where\footnote{Notice that $\ddot{\myvec q}^{T}\mymatrix J_{\dq{\xi}_{0,c_{i}}^{c_{i}}}^{T}\mymatrix{\Psi}_{c_{i}}\dot{\mymatrix J}_{\dq{\xi}_{0,c_{i}}^{c_{i}}}\dot{\myvec q}=\dot{\myvec q}^{T}\dot{\mymatrix J}_{\dq{\xi}_{0,c_{i}}^{c_{i}}}^{T}\mymatrix{\Psi}_{c_{i}}\mymatrix J_{\dq{\xi}_{0,c_{i}}^{c_{i}}}\ddot{\myvec q}$
and $\ddot{\myvec q}^{T}\mymatrix J_{\dq{\xi}_{0,c_{i}}^{c_{i}}}^{T}\mymatrix{\Psi}_{c_{i}}\bar{\myvec a}_{c_{i}}=\bar{\myvec a}_{c_{i}}^{T}\mymatrix{\Psi}_{c_{i}}\mymatrix J_{\dq{\xi}_{0,c_{i}}^{c_{i}}}\ddot{\myvec q}$.} $\mathcal{G}_{a_{i}}\left(\myvec q,\dot{\myvec q},\ddot{\myvec q}\right)\triangleq\frac{1}{2}\ddot{\myvec q}^{T}\mymatrix J_{\dq{\xi}_{0,c_{i}}^{c_{i}}}^{T}\mymatrix{\Psi}_{c_{i}}\mymatrix J_{\dq{\xi}_{0,c_{i}}^{c_{i}}}\ddot{\myvec q}+\dot{\myvec q}^{T}\dot{\mymatrix J}_{\dq{\xi}_{0,c_{i}}^{c_{i}}}^{T}\mymatrix{\Psi}_{c_{i}}\mymatrix J_{\dq{\xi}_{0,c_{i}}^{c_{i}}}\ddot{\myvec q}-\ddot{\myvec q}^{T}\mymatrix J_{\dq{\xi}_{0,c_{i}}^{c_{i}}}^{T}\mymatrix{\Psi}_{c_{i}}\bar{\myvec a}_{c_{i}}$
and $\mathcal{G}_{b_{i}}\left(\myvec q,\dot{\myvec q}\right)\triangleq\frac{1}{2}\bar{\myvec a}_{c_{i}}^{T}\mymatrix{\Psi}_{c_{i}}\bar{\myvec a}_{c_{i}}+\frac{1}{2}\dot{\myvec q}^{T}\dot{\mymatrix J}_{\dq{\xi}_{0,c_{i}}^{c_{i}}}^{T}\mymatrix{\Psi}_{c_{i}}\dot{\mymatrix J}_{\dq{\xi}_{0,c_{i}}^{c_{i}}}\dot{\myvec q}-\dot{\myvec q}^{T}\dot{\mymatrix J}_{\dq{\xi}_{0,c_{i}}^{c_{i}}}^{T}\mymatrix{\Psi}_{c_{i}}\bar{\myvec a}_{c_{i}}.$

From the optimality condition, the solution of \eqref{eq:GaussPrinciple}
is computed as \cite{Wieber2006}
\begin{align}
\frac{\partial\mathcal{G}\left(\myvec q,\dot{\myvec q},\ddot{\myvec q}\right)}{\partial\ddot{\myvec q}} & =\frac{\partial}{\partial\ddot{\myvec q}}\left(\underset{i=1}{\sum^{n}}\mathcal{G}_{a_{i}}\left(\myvec q,\dot{\myvec q},\ddot{\myvec q}\right)\right)=\myvec 0_{1\times n}.\label{eq:optimalityCond}
\end{align}

Using \eqref{eq:ObjectiveFunctionRo} in \eqref{eq:optimalityCond},
we obtain
\begin{gather}
\myvec 0_{n\times1}=\underset{i=1}{\sum^{n}}\left(\mymatrix J_{\dq{\xi}_{0,c_{i}}^{c_{i}}}^{T}\mymatrix{\Psi}_{c_{i}}\mymatrix J_{\dq{\xi}_{0,c_{i}}^{c_{i}}}\ddot{\myvec q}+\mymatrix J_{\dq{\xi}_{0,c_{i}}^{c_{i}}}^{T}\mymatrix{\Psi}_{c_{i}}\dot{\mymatrix J}_{\dq{\xi}_{0,c_{i}}^{c_{i}}}\dot{\myvec q}+\Phi\right),\label{eq:JacZero}
\end{gather}
where $\Phi\triangleq-\mymatrix J_{\dq{\xi}_{0,c_{i}}^{c_{i}}}^{T}\mymatrix{\Psi}_{c_{i}}\bar{\myvec a}_{c_{i}}.$

Since $\mymatrix J_{\dq{\xi}_{0,c_{i}}^{c_{i}}}=\left[\begin{array}{cc}
\mymatrix J_{\getp{\dq{\xi}_{0,c_{i}}^{c_{i}}}}^{T} & \mymatrix J_{\getd{\dq{\xi}_{0,c_{i}}^{c_{i}}}}^{T}\end{array}\right]^{T}$, using \eqref{eq.UnconstrainedAcceleration} and the elements $\bar{\mymatrix{\mathbb{I}}}_{i}^{c_{i}}$
and $m_{i}$ of $\mymatrix{\Psi}_{c_{i}}$, the term $\Phi$ from
\eqref{eq:JacZero} can be rewritten as
\begin{multline}
\Phi=-\mymatrix J_{\getp{\dq{\xi}_{0,c_{i}}^{c_{i}}}}^{T}\bar{\mymatrix{\mathbb{I}}}_{i}^{c_{i}}\vector_{3}\dot{\quat{\omega}}_{0,c_{i}}^{c_{i}}-\mymatrix J_{\getd{\dq{\xi}_{0,c_{i}}^{c_{i}}}}^{T}\vector_{3}\quat f_{0,c_{i}}^{c_{i}}\\
-m_{i}\mymatrix J_{\getd{\dq{\xi}_{0,c_{i}}^{c_{i}}}}^{T}\vector_{3}\left(\dot{\quat p}_{0,c_{i}}^{c_{i}}\times\quat{\omega}_{0,c_{i}}^{c_{i}}\right),\label{eq:Phi}
\end{multline}
where $\vector_{3}\left(\dot{\quat p}_{0,c_{i}}^{c_{i}}\times\quat{\omega}_{0,c_{i}}^{c_{i}}\right)=-\mymatrix S\left(\quat{\omega}_{0,c_{i}}^{c_{i}}\right)\vector_{3}\dot{\quat p}_{0,c_{i}}^{c_{i}}$,
with $\vector_{3}\dot{\quat p}_{0,c_{i}}^{c_{i}}{=\,}\mymatrix J_{\getd{\dq{\xi}_{0,c_{i}}^{c_{i}}}}\dot{\myvec q}$,
$\quat f_{0,c_{i}}^{c_{i}}{=\,}m_{i}\ddot{\quat p}_{0,c_{i}}^{c_{i}}$,
and $\mymatrix S\left(\cdot\right){\in\,}\mathrm{so(3)}$ is the skew-symmetric
matrix used as an operator that performs the cross-product \cite{spong2006robot}.

Furthermore, as $\vector_{3}\quat{\omega}_{0,c_{i}}^{c_{i}}=\mymatrix J_{\getp{\dq{\xi}_{0,c_{i}}^{c_{i}}}}\dot{\myvec q}$,
for convenience's sake we use the $\vector_{3}$ operator to rewrite
\eqref{eq:euler_eq} as
\begin{equation}
\bar{\mymatrix{\mathbb{I}}}_{i}^{c_{i}}\vector_{3}\dot{\quat{\omega}}_{0,c_{i}}^{c_{i}}{=}\vector_{3}\quat{\tau}_{0,c_{i}}^{c_{i}}{+}\mymatrix S\left(\myvec s_{c_{i}}\right)\mymatrix J_{\getp{\dq{\xi}_{0,c_{i}}^{c_{i}}}}\dot{\myvec q},\label{eq:Euler_eq2}
\end{equation}
where $\myvec s_{c_{i}}\triangleq\bar{\mymatrix{\mathbb{I}}}_{i}^{c_{i}}\vector_{3}\quat{\omega}_{0,c_{i}}^{c_{i}}$,
and use it in \eqref{eq:Phi} to obtain

\begin{multline}
\Phi=-\mymatrix J_{\getp{\dq{\xi}_{0,c_{i}}^{c_{i}}}}^{T}\vector_{3}\quat{\tau}_{0,c_{i}}^{c_{i}}-\mymatrix J_{\getd{\dq{\xi}_{0,c_{i}}^{c_{i}}}}^{T}\vector_{3}\quat f_{0,c_{i}}^{c_{i}}\\
+\mymatrix J_{\dq{\xi}_{0,c_{i}}^{c_{i}}}^{T}\overline{\mymatrix S}\left(\quat{\omega}_{0,c_{i}}^{c_{i}},\mymatrix{\Psi}_{c_{i}}\right)\mymatrix J_{\dq{\xi}_{0,c_{i}}^{c_{i}}}\dot{\myvec q},\label{eq:Phi_final}
\end{multline}
with
\begin{equation}
\overline{\mymatrix S}\left(\quat{\omega}_{0,c_{i}}^{c_{i}},\mymatrix{\Psi}_{c_{i}}\right)\triangleq\textrm{blkdiag\ensuremath{\left(-\mymatrix S\left(\myvec s_{c_{i}}\right),m_{i}\mymatrix S\left(\quat{\omega}_{0,c_{i}}^{c_{i}}\right)\right)}}.\label{eq:SkewSymetricMatrix}
\end{equation}

Finally, using \eqref{eq:Phi_final} in \eqref{eq:JacZero} yields
\begin{align}
\mymatrix M_{\text{GP}}\ddot{\myvec q}+\mymatrix C_{\text{GP}}\dot{\myvec q} & =\myvec{\bar{\tau}}_{\mathrm{GP}},\label{eq:EulerLagrangeEq}
\end{align}
where $\mymatrix M_{\text{GP}}\triangleq\mymatrix M_{\text{GP}}\left(\myvec q\right)\in\mathbb{R}^{n\times n}$
is the inertia matrix, $\mymatrix C_{\text{GP}}\triangleq\mymatrix C_{\text{GP}}\left(\myvec q,\dot{\myvec q}\right)\in\mathbb{R}^{n\times n}$
denotes the nonlinear dynamic effects (including the Coriolis terms),
and $\myvec{\bar{\tau}}_{\mathrm{GP}}\triangleq\myvec{\bar{\tau}}_{\mathrm{GP}}\left(\myvec q\right)\in\mathbb{R}^{n}$
represents the generalized forces acting on the system; also, 
\begin{align}
\mymatrix M_{\text{GP}} & \triangleq\underset{i=1}{\sum^{n}}\mymatrix J_{\dq{\xi}_{0,c_{i}}^{c_{i}}}^{T}\mymatrix{\Psi}_{c_{i}}\mymatrix J_{\dq{\xi}_{0,c_{i}}^{c_{i}}},\label{eq:InertiaMatrix}\\
\mymatrix C_{\text{GP}} & \triangleq\underset{i=1}{\sum^{n}}\mymatrix J_{\dq{\xi}_{0,c_{i}}^{c_{i}}}^{T}\left(\overline{\mymatrix S}\left(\quat{\omega}_{0,c_{i}}^{c_{i}},\mymatrix{\Psi}_{c_{i}}\right)\mymatrix J_{\dq{\xi}_{0,c_{i}}^{c_{i}}}+\mymatrix{\Psi}_{c_{i}}\dot{\mymatrix J}_{\dq{\xi}_{0,c_{i}}^{c_{i}}}\right),\label{eq:CoriolisMatrix}\\
\myvec{\bar{\tau}}_{\mathrm{GP}} & \triangleq\underset{i=1}{\sum^{n}}\mymatrix J_{\dq{\xi}_{0,c_{i}}^{c_{i}}}^{T}\myvec{\bar{\varsigma}}_{c_{i}},\label{eq:GeneralizedForcesVector}
\end{align}
where $\myvec{\bar{\varsigma}}_{c_{i}}$ is the wrench at the $i$th
center of mass, defined as
\begin{equation}
\myvec{\bar{\varsigma}}_{c_{i}}\triangleq\begin{bmatrix}\mymatrix 0_{3\times3} & \mymatrix I_{3\times3}\\
\mymatrix I_{3\times3} & \mymatrix 0_{3\times3}
\end{bmatrix}\vector_{6}\dq{\zeta}_{0,c_{i}}^{c_{i}},\label{eq:torques_c_i}
\end{equation}
with $\dq{\zeta}_{0,c_{i}}^{c_{i}}=\quat f_{0,c_{i}}^{c_{i}}+\varepsilon\quat{\tau}_{0,c_{i}}^{c_{i}}$.

Furthermore, since the gravity does not generate any resultant moment
at the center of mass of a link, the vector of gravitational forces
$\myvec{\tau}_{g}\triangleq\myvec{\tau}_{g}\left(\myvec q\right)$
is obtained from $\myvec{\bar{\tau}}_{\mathrm{GP}}$ by letting $\quat{\tau}_{0,c_{i}}^{c_{i}}=0$
and $\quat f_{0,c_{i}}^{c_{i}}=\ad{\quat r_{0}^{c_{i}}}{\quat f_{g_{i}}}$,
where $\quat f_{g_{i}}=m_{i}\myvec g$ and $\quat g\in\mathbb{H}_{p}$
are the gravitational force and gravitational acceleration, respectively,
both expressed in the inertial frame. Hence,
\begin{align}
\myvec{\tau}_{g} & =\underset{i=1}{\sum^{n}}\mymatrix J_{\getd{\dq{\xi}_{0,c_{i}}^{c_{i}}}}^{T}\vector_{3}\left(\ad{\quat r_{0}^{c_{i}}}{\quat f_{g_{i}}}\right).\label{eq:GravityForcesVector}
\end{align}

By considering the generalized forces $\myvec{\tau}_{\textrm{GP}}$
applied at the joints and the gravitational forces $\myvec{\tau}_{g}$,
the resultant forces acting on the system are $\myvec{\bar{\tau}}_{\mathrm{GP}}=\myvec{\tau}_{\textrm{GP}}+\myvec{\tau}_{g}$.
Let $\myvec g_{\textrm{GP}}\triangleq-\myvec{\tau}_{g}$, then \eqref{eq:EulerLagrangeEq}
is rewritten in the canonical form as

\begin{align}
\mymatrix M_{\text{GP}}\ddot{\myvec q}+\mymatrix C_{\text{GP}}\dot{\myvec q}+\myvec g_{\textrm{GP}} & =\myvec{\tau}_{\textrm{GP}}.\label{eq:CanonicalEulerLagrange}
\end{align}
In this way, solving \eqref{eq:GaussPrinciple} leads to the Euler-Lagrange
dynamic description of a mechanical system by means of dual quaternion
algebra. Once again, we assume that the robot forward kinematics and
differential kinematics are available in dual quaternion space \cite{Adorno2011}.

\begin{remark}

Let $\mymatrix A\triangleq\left(1/2\right)\dot{\mymatrix M}_{\text{GP}}-\mymatrix C_{\text{GP}}$,
then 
\begin{gather*}
\mymatrix A=-\underset{i=1}{\sum^{n}}\mymatrix J_{\dq{\xi}_{0,c_{i}}^{c_{i}}}^{T}\overline{\mymatrix S}\left(\quat{\omega}_{0,c_{i}}^{c_{i}},\mymatrix{\Psi}_{c_{i}}\right)\mymatrix J_{\dq{\xi}_{0,c_{i}}^{c_{i}}.}
\end{gather*}
Since $\overline{\mymatrix S}\left(\quat{\omega}_{0,c_{i}}^{c_{i}},\mymatrix{\Psi}_{c_{i}}\right)$
is skew-symmetric by construction, then $\mymatrix A^{T}=\mymatrix{-A}$,
which implies
\begin{align}
\myvec u^{T}\left(\frac{1}{2}\dot{\mymatrix M}_{\text{GP}}\left(\myvec q\right)-\mymatrix C_{\text{GP}}\left(\myvec q,\dot{\myvec q}\right)\right)\myvec u & =0\label{eq:SkewProperty}
\end{align}
for all $\myvec q,\dot{\myvec q},\myvec u\in\mathbb{R}^{n}$. Property
\eqref{eq:SkewProperty} is useful to show formal closed-loop stability
in robot dynamic control using strategies based on Lyapunov functions
\cite{Kelly2005}.

\end{remark}

\subsection{Connections with the Gibbs-Appell and Kane's equations}

The Gibbs-Appell and Kane's equations have proven to be a powerful
mathematical tool to describe both unconstrained and constrained mechanical
systems without the use of Lagrange multipliers \cite{Storch1989,Honein2021}.
Both are different ways to get the equations of motion, but equivalent
in the sense that a set of equations implies the other.\cite{Townsend1992,Desloge1987_,Levinson87}.

The Gibbs-Appell method is closely related with the Gauss's Principle
of Least Constraint, since both approaches use scalar quadratic functions
in terms of accelerations. The former can be seen as a generalization
of the latter \cite{Ray1972,JRay92}. However, they are equivalent
and both can be derived from the other \cite{Ray1972,Lewis1996,Udwadia1998}.
Nonetheless, different from the Gibbs-Appell and Kane's equations,
the Gauss's principle strategy does not require setting up quasi-velocities
and allows taking into account additional constraints directly in
the optimization formulation.

Now, we rewrite the Gibbs-Appell and Kane's equations using the equations
derived in Sections~\ref{subsec:Constrained-acceleration}--\ref{subsec:Euler-Lagrange-equations}.
Furthermore, we show that the Euler-Lagrange dynamic description of
a mechanical system can be shown to be a particular case of the Gibbs-Appell
and Kane's equations. This is done by selecting the quasi-velocities
to be the same as the generalized velocities.

For $n$ rigid bodies, the Gibbs-Appell equations are given by \cite{Honein2021}
\begin{equation}
\frac{\partial S\left(\myvec q,\dot{\myvec q},\ddot{\myvec q}\right)}{\partial\dot{\myvec u}}=\underset{\myvec{\bar{\tau}}_{\mathrm{GA}}}{\underbrace{\underset{i=1}{\sum^{n}}\left(\frac{\partial}{\partial\myvec u}\myvec{\nu}_{c_{i}}\right)^{T}\myvec{\bar{\varsigma}}_{c_{i}}}},\label{eq:gibbs_appell_init}
\end{equation}
where 
\begin{equation}
S\left(\myvec q,\dot{\myvec q},\ddot{\myvec q}\right)\triangleq\underset{i=1}{\sum^{n}}\left(\frac{1}{2}\myvec a_{c_{i}}^{T}\mymatrix{\Psi}_{c_{i}}\myvec a_{c_{i}}+\myvec a_{c_{i}}^{T}\overline{\mymatrix S}\left(\quat{\omega}_{0,c_{i}}^{c_{i}},\mymatrix{\Psi}_{c_{i}}\right)\myvec{\nu}_{c_{i}}\right)\label{eq:Gibbs-Abbel scalar.}
\end{equation}
is a scalar function of the configuration $\myvec q$, configuration
velocity $\dot{\myvec q}$, and configuration acceleration $\ddot{\myvec q}$.
The vector $\myvec{\bar{\tau}}_{\mathrm{GA}}$ contains generalized
forces associated with the quasi-velocities $\myvec u$. Furthermore,
$\myvec{\nu}_{c_{i}}$ and $\myvec{\bar{\varsigma}}_{c_{i}}$ are
the twist and the generalized forces of the $i$th body, given by
\eqref{eq:ConstrainedVel} and \eqref{eq:torques_c_i}, respectively.

We let $\myvec u\triangleq\dot{\myvec q}$, and use \eqref{eq:ConstrainedAcceleration}
in \eqref{eq:Gibbs-Abbel scalar.}. We take the result and apply the
partial derivative $\partial S/\partial\ddot{\myvec q}$, and then
compare with \eqref{eq:InertiaMatrix} and \eqref{eq:CoriolisMatrix}
to obtain
\begin{equation}
\frac{\partial S\left(\myvec q,\dot{\myvec q},\ddot{\myvec q}\right)}{\partial\ddot{\myvec q}}=\mymatrix M_{\text{GP}}\ddot{\myvec q}+\mymatrix C_{\text{GP}}\dot{\myvec q}.\label{eq:gibbs_appell_right}
\end{equation}

Using \eqref{eq:ConstrainedVel}, we obtain
\begin{equation}
\frac{\partial}{\partial\myvec u}\myvec{\nu}_{c_{i}}=\frac{\partial}{\partial\dot{\myvec q}}\left(\mymatrix J_{\dq{\xi}_{0,c_{i}}^{c_{i}}}\dot{\myvec q}\right)=\mymatrix J_{\dq{\xi}_{0,c_{i}}^{c_{i}}}.\label{eq:partial_velocities}
\end{equation}

Therefore, we compute the generalized forces as $\myvec{\bar{\tau}}_{\mathrm{GA}}=\sum_{i=1}^{n}\mymatrix J_{\dq{\xi}_{0,c_{i}}^{c_{i}}}^{T}\myvec{\bar{\varsigma}}_{c_{i}}=\myvec{\bar{\tau}}_{\mathrm{GP}}$
(see \eqref{eq:GeneralizedForcesVector}). In this way, solving \eqref{eq:gibbs_appell_init}
leads to the Euler-Lagrange equation, which is identical to the one
obtained by using the Gauss's principle \eqref{eq:EulerLagrangeEq}.

On the other hand, the Kane's equation of motion is given by \cite{Kane1983Dynamics}
\begin{equation}
\myvec{\varphi}-\bar{\myvec{\varphi}}=\myvec 0,\label{eq:Kane_equations}
\end{equation}
where $\myvec{\varphi}$ contains the generalized active forces and
$\bar{\myvec{\varphi}}$ contains the generalized inertia forces.

We obtain the equation of motion using Kane's method by grouping
both the Newton's and Euler's equations for $n$ rigid bodies as follows\footnote{The the Newton's equation is given by $m_{i}\vector_{3}\ddot{\quat p}_{0,c_{i}}^{c_{i}}{=}\vector_{3}\quat f_{0,c_{i}}^{c_{i}}$
and the Euler's equation is given by \eqref{eq:Euler_eq2}.}
\begin{equation}
\underset{i=1}{\sum^{n}}\left[\begin{array}{c}
\left(\bar{\mymatrix{\mathbb{I}}}_{i}^{c_{i}}\vector_{3}\dot{\quat{\omega}}_{0,c_{i}}^{c_{i}}-\mymatrix S\left(\myvec s_{c_{i}}\right)\vector_{3}\quat{\omega}_{0,c_{i}}^{c_{i}}\right)\\
m_{i}\vector_{3}\ddot{\quat p}_{0,c_{i}}^{c_{i}}
\end{array}\right]=\underset{i=1}{\sum^{n}}\underset{\myvec{\bar{\varsigma}}_{c_{i}}}{\underbrace{\left[\begin{array}{c}
\vector_{3}\quat{\tau}_{0,c_{i}}^{c_{i}}\\
\vector_{3}\quat f_{0,c_{i}}^{c_{i}}
\end{array}\right]}}.\label{eq:FundamentalEquations-1}
\end{equation}

Using \eqref{eq:SkewSymetricMatrix}, and the fact that $\myvec{\nu}_{c_{i}}=\vector_{6}\dq{\xi}_{0,c_{i}}^{c_{i}}$,with
$\dq{\xi}_{0,c_{i}}^{c_{i}}\triangleq\dq{\xi}_{0,c_{i}}^{c_{i}}\left(\myvec q\right)$
given as in \eqref{eq:twist-body-frame}, $\myvec a_{c_{i}}=\vector_{6}\dot{\dq{\xi}}_{0,c_{i}}^{c_{i}}$,
with $\dot{\dq{\xi}}_{0,c_{i}}^{c_{i}}\triangleq\dot{\dq{\xi}}_{0,c_{i}}^{c_{i}}\left(\myvec q,\dot{\myvec q}\right)$
given as in \eqref{eq.UnconstrainedAcceleration},\footnote{Notice that, although $\myvec a_{c_{i}}$ is analogous to \eqref{eq.UnconstrainedAcceleration},
it refers to the actual accelerations and, therefore, the constrained
ones. Consequently, $\myvec a_{c_{i}}$ depends on $\myvec q$ and
$\dot{\myvec q}$, whereas \eqref{eq.UnconstrainedAcceleration} does
not.} and $\vector_{3}\left(\dot{\quat p}_{0,c_{i}}^{c_{i}}\times\quat{\omega}_{0,c_{i}}^{c_{i}}\right)=-\mymatrix S\left(\quat{\omega}_{0,c_{i}}^{c_{i}}\right)\vector_{3}\dot{\quat p}_{0,c_{i}}^{c_{i}}$,
we rewrite \eqref{eq:FundamentalEquations-1} as
\begin{equation}
\underset{i=1}{\sum^{n}}\underset{\kappa_{i}}{\underbrace{\left(\mymatrix{\Psi}_{c_{i}}\myvec a_{c_{i}}+\overline{\mymatrix S}\left(\quat{\omega}_{0,c_{i}}^{c_{i}},\mymatrix{\Psi}_{c_{i}}\right)\myvec{\nu}_{c_{i}}\right)}}=\underset{i=1}{\sum^{n}}\myvec{\bar{\varsigma}}_{c_{i}}.\label{eq:FundamentalEquations}
\end{equation}

Since $\kappa_{i}=\bar{\varsigma}_{c_{i}}$ for $i\in\{1,\ldots,n\}$,
we multiply each of the $n$ terms $\kappa_{i}$ and $\bar{\varsigma}_{c_{i}}$
from \eqref{eq:FundamentalEquations} by $\left(\frac{\partial}{\partial\myvec u}\myvec{\nu}_{c_{i}}\right)^{T}=\mymatrix J_{\dq{\xi}_{0,c_{i}}^{c_{i}}}^{T},$
to obtain the Kane's equations \cite{Kane1983Dynamics}, which yields
\begin{align}
\underset{\myvec{\varphi}}{\underbrace{\underset{i=1}{\sum^{n}}\mymatrix J_{\dq{\xi}_{0,c_{i}}^{c_{i}}}^{T}\left(\mymatrix{\Psi}_{c_{i}}\myvec a_{c_{i}}+\overline{\mymatrix S}\left(\quat{\omega}_{0,c_{i}}^{c_{i}},\mymatrix{\Psi}_{c_{i}}\right)\myvec{\nu}_{c_{i}}\right)}} & =\underset{\bar{\myvec{\varphi}}}{\underbrace{\underset{i=1}{\sum^{n}}\mymatrix J_{\dq{\xi}_{0,c_{i}}^{c_{i}}}^{T}\myvec{\bar{\varsigma}}_{c_{i}}}}.\label{eq:KaneEquations}
\end{align}

Finally, using \eqref{eq:ConstrainedVel}, \eqref{eq:ConstrainedAcceleration},
\eqref{eq:InertiaMatrix}, \eqref{eq:CoriolisMatrix} and \eqref{eq:GeneralizedForcesVector}
in \eqref{eq:KaneEquations} we obtain $\myvec{\varphi}=\mymatrix M_{\text{GP}}\ddot{\myvec q}+\mymatrix C_{\text{GP}}\dot{\myvec q}$
and $\bar{\myvec{\varphi}}=\myvec{\bar{\tau}}_{\mathrm{GP}}$, which
is the same dynamic equation as the one obtained by the Gauss's principle
\eqref{eq:EulerLagrangeEq}, as expected.

Therefore, when considering the quasi-velocities to be the same as
the generalized velocities (i.e., $\myvec u\triangleq\dot{\myvec q}$)
the relations between Gauss's principle, Gibbs-Appell equations and
Kane's method are given by
\begin{equation}
\frac{\partial\mathcal{G}\left(\myvec q,\dot{\myvec q},\ddot{\myvec q}\right)}{\partial\ddot{\myvec q}}=\frac{\partial S\left(\myvec q,\dot{\myvec q},\ddot{\myvec q}\right)}{\partial\ddot{\myvec q}}-\myvec{\bar{\tau}}_{\mathrm{GA}}=\myvec{\varphi}-\bar{\myvec{\varphi}}=\myvec 0_{n\times1}.\label{eq:EquivalenceGibbsGauss}
\end{equation}

\subsection{Constrained Robotic Systems using the Gauss's Principle of Least
Constraint\label{subsec:Constrained-Robotic-Systems}}

Additional constraints can be imposed in the GPLC formulation. This
can be done by means of Lagrange multipliers \cite{Bouyarmane2012}
or using the Udwadia-Kalaba formulation \cite{FirdausE.UdwadiaandRobertE.Kalaba1992}.
The former requires the computation of the Lagrange multipliers, whereas
the latter employs a simpler method, albeit equivalent, which is based
on generalized inverses as the solution to a constrained quadratic
program.

Using the Udwadia-Kalaba formulation \cite{FirdausE.UdwadiaandRobertE.Kalaba1992},
additional constraints in the form $\mymatrix A\ddot{\myvec q}=\myvec b$,
with $\mymatrix A\triangleq\mymatrix A(\myvec q,\dot{\myvec q})$
and $\myvec b\triangleq\myvec b(\myvec q,\dot{\myvec q})$, are taken
into account in \eqref{eq:EulerLagrangeEq} as follows
\begin{align}
\mymatrix M_{\textrm{GP}}\ddot{\myvec q} & =\mymatrix Q\left(\myvec q,\dot{\myvec q}\right)+\text{\ensuremath{\mymatrix Q}}_{c}\left(\myvec q,\dot{\myvec q}\right),\label{eq:UKformulation}
\end{align}
where $\mymatrix Q\left(\myvec q,\dot{\myvec q}\right)\triangleq\myvec{\bar{\tau}}_{\mathrm{GP}}-\mymatrix C_{\text{GP}}\dot{\myvec q}$,
and the additional term $\text{\ensuremath{\mymatrix Q}}_{c}\left(\myvec q,\dot{\myvec q}\right)\triangleq\mymatrix M_{\textrm{GP}}^{1/2}\mymatrix D^{+}\left(\myvec b-\mymatrix A\mymatrix M_{\textrm{GP}}^{-1}\mymatrix Q\left(\myvec q,\dot{\myvec q}\right)\right)$
represents the constraint force due to the additional constraints.
Furthermore, $\mymatrix D^{+}$ is the Moore-Penrose generalized inverse
\cite{spong2006robot} of $\mymatrix D$, with $\mymatrix D\triangleq\mymatrix A\mymatrix M_{\textrm{GP}}^{-1/2}$.

Finally, letting $\mymatrix{\Omega}\triangleq\left(\mymatrix I-\mymatrix M_{\textrm{GP}}^{1/2}\mymatrix D^{+}\mymatrix A\mymatrix M_{\textrm{GP}}^{-1}\right)$,
we rewrite \eqref{eq:UKformulation} as
\begin{equation}
\mymatrix M_{\textrm{GP}}\ddot{\myvec q}+\mymatrix{\Omega}\mymatrix C_{\textrm{GP}}\dot{\myvec q}-\mymatrix M_{\textrm{GP}}^{1/2}\mymatrix D^{+}\myvec b=\mymatrix{\Omega}\myvec{\bar{\tau}}_{\mathrm{GP}}.\label{eq:FinalGaussUK}
\end{equation}

For example, consider the well-known differential-drive mobile robot,
in which the nonholonomic constraint ensures the conditions of pure
rolling and non-slipping movements \cite{Fierro1997_}. The robot
configuration is specified by the vector $\myvec q=\left[\begin{array}{ccc}
x & y & \phi\end{array}\right]^{T}$, where $x,y$ is the position coordinates and $\phi$ is the orientation
of the robot on the plane. The nonholonomic constraint is given by
\begin{equation}
\underset{\mymatrix A}{\underbrace{\begin{bmatrix}-\sin\phi & \cos\phi & 0\end{bmatrix}}}\dot{\myvec q}=0,\label{eq:first-order-nonholonomic-constraint}
\end{equation}
which can be enforced in \eqref{eq:FinalGaussUK} by taking the time
derivative of \eqref{eq:first-order-nonholonomic-constraint} such
that $\dot{\mymatrix A}\dot{\myvec q}+\mymatrix A\ddot{\myvec q}=0$.
Therefore, $\myvec b=-\dot{\mymatrix A}\dot{\myvec q}$.

\section{Results \label{sec:Results}}

To assess the dual quaternion Newton-Euler formalism (dqNE) and the
Euler-Lagrange model obtained using the dual quaternion Gauss's Principle
of Least Constraint (dqGP), we performed simulations using three different
robots; namely, a fixed-base $50$-DoF serial manipulator, a $9$-DoF
holonomic mobile manipulator, and an $8$-DoF nonholonomic mobile
manipulator.

We implemented the simulations on the robot simulator ${\text{V-REP PRO EDU V3.6.2}}$\footnote{Available at: https://www.coppeliarobotics.com/}
using an interface with Matlab 2020a and the computational library
DQ Robotics \cite{AdornoDQRobotics2020} for dual quaternion algebra
on a computer running Ubuntu 18.04 LTS 64 bits equipped with a Intel
Core i7 6500U with 8GB RAM.

Furthermore, we present the computational costs of the proposed methods
and compare them with their respective classic counterparts.

\subsection{Simulation Setup}

The comparisons between the generalized accelerations obtained through
our proposed models and the values from V-REP were made considering
the coefficient of multiple correlation (CMC) \cite{Ferrari2010}
between the waveforms. The CMC provides a coefficient ranging between
zero and one that indicates how similar two given waveforms are. Identical
waveforms have CMC equal to one, whereas completely different waveforms
have CMC equal to zero.

The simulator does not allow the direct reading of accelerations.
Therefore, to obtain the configuration acceleration vector $\ddot{\myvec q}$,
we first read the velocity vector $\dot{\myvec q}\in\mathbb{R}^{n_{\ell}+3}$,
then filtered all elements $\dot{q}_{1},\ldots,\dot{q}_{n_{\ell}+3}$
with a discrete low-pass Butterworth filter, and used those values
to obtain the accelerations by means of numerical differentiation
based on a second forward finite difference approximation. We then
calculated the CMC between each element $\ddot{q}_{i}$, with $i\in\{1,\ldots,n_{\ell}+3\}$,
of the generalized acceleration waveform and its counterpart from
V-REP. Afterward, we used those CMCs to obtain the mean, minimum,
and maximum CMCs for the model, alongside their standard deviation.

Furthermore, for the simulation of the fixed-base $50$-DoF serial
manipulator, we also used the classic Newton-Euler algorithm (rtNE)
implemented on the Robotics Toolbox \cite{Corke2017}, and calculated
the CMC between the joint acceleration waveforms yielded by it and
the ones from V-REP.\footnote{In this case, we used the Vortex Studio engine (www.cm-labs.com) because
it presented better numerical stability for the 50-DoF manipulator.} The Robotics Toolbox is a widely used library whose accuracy has
been verified throughout the years. Therefore, it is an appropriate
baseline for the evaluation of the CMCs obtained by using our models.

\subsection{Results}

Table~\ref{tab:cmc} presents the CMC between the generalized acceleration
waveforms obtained through the different dynamic model strategies
(dqNE, dqGP, and rtNE) and the values obtained from V-REP.

The proposed dual quaternion Newton-Euler formalism does not allow
the inclusion of equality constraints into the model; therefore, it
was not applied to the dynamic modeling of the nonholonomic mobile
manipulator. Similarly, the current version of the Robotics Toolbox
only supports dynamic modeling of fixed-base robots and was only applied
to the $50$-DoF serial manipulator. The cases where the model could
not be obtained using the listed strategy are indicated in Table~\ref{tab:cmc}
by N/A. (i.e., not available). For all other cases, all models presented
mean ($\mathrm{mean}$) and minimum ($\mathrm{min}$) CMC close to
one, with small standard deviation ($\mathrm{std}$) and high maximum
$\left(\mathrm{max}\right)$ CMC; thus indicating high similarity
between the generalized acceleration waveform obtained from them and
the values from V-REP.

\begin{table*}[t]
\caption{CMC between the joint acceleration waveforms obtained through different
dynamic model strategies and the values obtained from V-REP. The closer
to one, the more similar the waveforms are.\label{tab:cmc}}

\centering{}{\small{}}%
\begin{tabular}{cc|cccc|cccc|cccc}
\hline 
 & \multicolumn{1}{c}{} & \multicolumn{4}{c}{{\small{}\hspace{1mm}50-DoF serial manipulator}} & \multicolumn{4}{c}{{\small{}9-DoF holonomic mobile manipulator}} & \multicolumn{4}{c}{{\small{}~~8-DoF nonholonomic mobile manipulator~~}}\tabularnewline
\hline 
{\small{}~~~~~Method~~~~~} &  & {\small{}\hspace{4mm}min\hspace{4mm}} & {\small{}\hspace{4mm}mean\hspace{4mm}} & {\small{}\hspace{4mm}std\hspace{4mm}} & {\small{}\hspace{4mm}max\hspace{4mm}} & {\small{}\hspace{4mm}min\hspace{4mm}} & {\small{}\hspace{4mm}mean\hspace{4mm}} & {\small{}\hspace{4mm}std\hspace{4mm}} & {\small{}\hspace{4mm}max\hspace{4mm}} & {\small{}\hspace{4mm}min\hspace{4mm}} & {\small{}\hspace{4mm}mean\hspace{4mm}} & {\small{}\hspace{4mm}std\hspace{4mm}} & {\small{}\hspace{4mm}max~~~~~}\tabularnewline
\hline 
\multicolumn{2}{c|}{{\small{}dqGP~~~}} & {\small{}$0.9044$} & {\small{}$0.9893$} & {\small{}$0.0182$} & {\small{}$0.9993$} & {\small{}$0.9934$} & {\small{}$0.9973$} & {\small{}$0.0026$} & {\small{}$0.9999$} & {\small{}$0.8860$} & {\small{}$0.9839$} & {\small{}$0.0368$} & {\small{}$0.9999$}\tabularnewline
\multicolumn{2}{c|}{{\small{}dqNE~~~}} & {\small{}$0.9044$} & {\small{}$0.9893$} & {\small{}$0.0182$} & {\small{}$0.9993$} & {\small{}$0.9938$} & {\small{}$0.9977$} & {\small{}$0.0022$} & {\small{}$0.9999$} & {\small{}N/A} & {\small{}N/A} & {\small{}N/A} & {\small{}N/A}\tabularnewline
\multicolumn{2}{c|}{{\small{}rtNE~~}} & {\small{}$0.9044$} & {\small{}$0.9893$} & {\small{}$0.0182$} & {\small{}$0.9993$} & {\small{}N/A} & {\small{}N/A} & {\small{}N/A} & {\small{}N/A} & {\small{}N/A} & {\small{}N/A} & {\small{}N/A} & {\small{}N/A}\tabularnewline
\hline 
\end{tabular}{\small\par}
\end{table*}

For the $50$-DoF serial manipulator, both the dqNE and the dqGP are
equivalent to the rtNE, which demonstrates the accuracy of our proposed
strategies when compared to the classic Newton-Euler approach.

For qualitative analysis, Fig.~\ref{fig:joint_acc_nonholonomic_jaco}
presents the generalized accelerations obtained using dqGP, alongside
the V-REP values, for the minimum, maximum, and intermediate CMCs
found during simulations. Even for the smallest value of CMC (i.e.,
0.8860), the accelerations obtained using our formulation match closely
the V-REP values. The small discrepancies arise from both discretization
effects and because the accelerations in V-REP are estimated from
noisy velocity values.

\begin{figure}[t]
\begin{centering}
{\Huge{}\resizebox{0.93\columnwidth}{!}{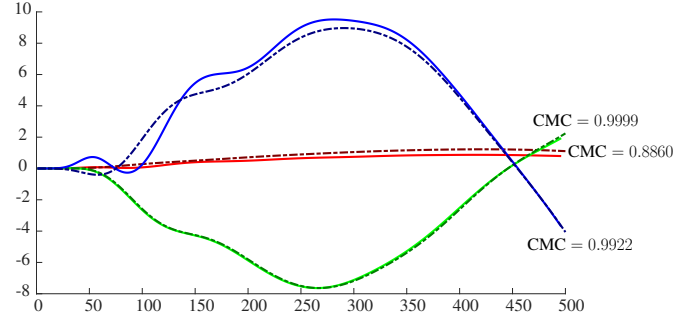}}{\Huge\par}
\par\end{centering}
\caption{Generalized acceleration waveforms of the 8-DoF nonholonomic mobile
manipulator. \emph{Solid} curves correspond to the V-REP values, whereas
\emph{dot-dashed} curves correspond to the values obtained using the
dqGP for the generalized acceleration waveforms of the first ($\text{CMC}=0.8860$),
ninth ($\text{CMC}=0.9999$), and fifth ($\text{CMC}=0.9922$) joints,
respectively.\label{fig:joint_acc_nonholonomic_jaco}}
\end{figure}

\subsection{Computational cost}

Here we compare the proposed methods with their classic counterparts
in terms of the number of multiplications and additions involved in
each technique. The results, considering an \emph{n}-DoF serial robot,
are summarized in Table~\ref{tab:summary_cost_table-1}. For the
classic Newton-Euler algorithm, we consider the version based on three
dimensional vectors proposed by Luh et al. \cite{Luh1980}, whose
mathematical cost was calculated by Balafoutis \cite{Balafoutis1994},
and is, to the best of our knowledge, one of the most efficient implementations
in the literature. Furthermore, for the classic Euler-Lagrange algorithm
we consider the version proposed by Hollerbach \cite{Hollerbach1980}.

\begin{table*}[t]
\noindent \centering{}\caption{Cost comparison between the proposed methods and their classic counterparts
for obtaining the dynamical model for an \emph{n}-DoF serial robot.
\label{tab:summary_cost_table-1}}
{\small{}}%
\begin{tabular*}{1\textwidth}{@{\extracolsep{\fill}}lcc}
\hline 
{\small{}Method} & {\small{}Mult.} & {\small{}Add.}\tabularnewline
\hline 
{\small{}\makecell[l]{Dual Quaternion Newton-Euler algorithm \\(cost
for arbitrary joints)}} & {\small{}$882n-48$} & {\small{}$724n-40$}\tabularnewline
{\small{}Classic Newton-Euler algorithm \cite{Balafoutis1994}} & {\small{}$150n-48$} & {\small{}$131n-48$}\tabularnewline
{\small{}\makecell[l]{Dual Quaternion Euler-Lagrange algorithm \\using
Gauss's Principle of Least Constraint}} & \multirow{1}{*}{{\small{}$4n^{3}+386n^{2}+401n$}} & \multirow{1}{*}{{\small{}$\frac{16}{3}n^{3}+326n^{2}+\frac{908}{3}n$}}\tabularnewline
{\small{}Classic Euler-Lagrange algorithm \cite{Hollerbach1980}} & \multirow{1}{*}{{\small{}$412n-277$}} & \multirow{1}{*}{{\small{}$320n-201$}}\tabularnewline
\hline 
\end{tabular*}
\end{table*}

The algorithm presented by Luh et al. \cite{Luh1980} costs less than
our Dual Quaternion Newton-Euler algorithm. The cost we presented
for our method is, however, fairly conservative and is given as an
upper bound. For instance, our calculations could be further optimized
by exploring the fact that several operations involve pure dual quaternions,
which have six elements instead of eight. Additionally, the cost presented
by Balafoutis \cite{Balafoutis1994} does not include the costs of
obtaining the robot kinematic model. Also, our method works for \emph{any}
type of joint and we have not optimized the calculations for any particular
type of joint, differently from Luh et al. \cite{Luh1980}, who only
consider prismatic and revolute joints, which are exploited to optimize
the computational cost. Nonetheless, both our algorithm and the one
of Luh et al. have linear costs in the number of DoF, with coefficients
of the same order of magnitude.

The Euler-Lagrange method based on the Gauss's Principle of Least
Constraint is, as expected, more expensive than the ones based on
the Newton-Euler and classic Euler-Lagrange formalism since it is
not based on recursive strategies. However, this strategy allows taking
into account additional constraints in the accelerations, which can
be exploited, for instance, in nonholonomic robotic systems. For those
cases, the Euler-Lagrange dynamic equation is given by \eqref{eq:FinalGaussUK}.

\section{Conclusions\label{sec:Conclusions}}

This work presented two strategies for the formulation of the dynamics
of mobile manipulators based on dual quaternion algebra. The first
one is based on the recursive Newton-Euler formulation, and uses twists
and wrenches instead of free vectors. This representation removes
the necessity of exhaustive geometrical analyses of the kinematic
chain since wrenches and twists are propagated through high-level
algebraic operations. Furthermore, our formulation works for any type
of joint because it takes into account arbitrary twists. Thus, our
strategy is more general than the work of Miranda et al. \cite{MirandadeFarias2019Journal},
which considered only manipulators with revolute joints.

The second proposed method is based on the Gauss's Principle of Least
Constraint and is also formulated based on twists and wrenches represented
using dual quaternion algebra in matrix form. This strategy allows
the incorporation of equality constraints directly in the optimization
formulation.

The cost comparison performed between the proposed methods and their
classic counterparts, in terms of number of multiplications and additions,
showed that the use of dual quaternions does not significatively increases
the cost of the Newton-Euler formalism, as the algorithm has linear
complexity on the number of rigid bodies in the kinematic chain. However,
the cost of obtaining the Euler-Lagrange model using the Gauss's Principle
of Least Constraint and dual quaternion algebra is higher than the
best classic Euler-Lagrange recursive solution found in the literature.
Notwithstanding, our method is far more general than its classic counterpart.
Also, we made no hard attempt, if any, to optimize our implementation
since we are currently more interested in the theoretical aspects
of the dynamic modeling using dual quaternion algebra than in ensuring
computational efficiency. In our current MATLAB implementation, the
dqNE and the dqGP take, in average, 23.17s and 8.73s to generate the
joints accelerations for a 50-DoF manipulator robot, respectively.
Those values are expected to decrease to around 99 ms and 37 ms in
a C++ implementation \cite{AdornoDQRobotics2020}, respectively.

Obtaining the Euler-Lagrange model through the Newton-Euler formalism
requires several executions of the algorithm. One execution to obtain
the gravitational vector, one to obtain the vector of Coriolis and
centrifugal terms and one for each row of the inertia matrix $\mymatrix M$.
For a 50-DoF manipulator robot, this results in $52$ executions of
the dqNE for each simulation step. For control applications, however,
we are usually interested in finding the joint torques,which requires
only one execution of the dqNE to generate each control input. Therefore,
using the dqNE to compute the joint torques for a 50-DoF manipulator
robot, the execution time is expected to reduce by a factor of $52$,
from around 99~ms to 1.9~ms in a C++ implementation.

Finally, we compared the joints accelerations obtained through the
proposed strategies for three different robots, with the values obtained
from ${\text{V-REP PRO EDU V3.6.2}}$, which is a realistic simulator.
The results showed that all of our methods are accurate for both fixed-base
serial manipulators and mobile manipulators.

Future works will focus on extending the dual quaternion Newton-Euler
algorithm to non-serial multibody systems (e.g., humanoids), and in
wrench control strategies. Concerning the Euler-Lagrange model obtained
using the Gauss's Principle of Least Constraint and dual quaternion
algebra, future works will be focused on exploiting inequality constraints
in the optimization formulation.

\section*{Acknowledgment}

This work was supported by the Brazilian agencies CAPES, CNPq (424011/2016-6
and 303901/2018-7), and by the INCT (National Institute of Science
and Technology) under the grant CNPq (Brazilian National Research
Council) 465755/2014-3.

We would like to thank our colleague Ana Christine de Oliveira for
providing us with the MATLAB implementation of the CMC used in Section~\ref{sec:Results}.

\bibliographystyle{asmems4}
\bibliography{AnalyticalDynamics,JMR-2019-Gauss,SumsofPowers,UK,Gibbs-Appell,gibbs_vs_appell,NumericalSolvers,Fierro_Lewis,JMR-2019}

\end{document}